\documentclass[10pt,twocolumn,romanappendices]{IEEEtran} 

\usepackage[utf8]{inputenc} 
\usepackage[intlimits]{amsmath}
\usepackage{amsfonts}%
\usepackage{amssymb}%
\usepackage[final]{graphicx}
\usepackage[noadjust]{cite}
\usepackage{setspace}
\usepackage{hyperref,url}
\usepackage{bm}
\usepackage{color}
\usepackage{booktabs}
\usepackage{mathtools}
\usepackage{dsfont}
\usepackage{algorithm}
\usepackage{algorithmic}
\usepackage{setspace}

\newtheorem{theorem}{Theorem}
\newtheorem{corollary}[theorem]{Corollary}

\newtheorem{definition}[theorem]{Definition}

\newtheorem{proof}{Proof}

\newcommand{\mybf}{\mathbf}

\def\I{\mathcal{I}}
\def\J{\mathcal{J}}
\def\cardI#1{\vert\mathcal{I}_{#1}\vert}
\def\T{\mathcal{T}}

\def\X{\mathcal{X}}
\def\Y{\mathcal{Y}}
\def\Mall{\widetilde{\mathbf{M}}}
\def\Wall{\widetilde{\mathbf{W}}}

\newcommand{\algorithmicinit}{\textbf{Initialization}}
\newcommand{\INIT}{\item[\algorithmicinit]}
\newcommand{\algorithmicinput}{\textbf{Input }}
\newcommand{\INPUT}{\item[\algorithmicinput]}
\newcommand{\algorithmicoutput}{\textbf{Output }}
\newcommand{\OUTPUT}{\item[\algorithmicoutput]}

\newcommand{\algorithmiciter}{\textbf{Iterative analysis }}
\newcommand{\ANALYSIS}{\item[\algorithmiciter]}

\begin{document}

\title{Data-Driven Tree Transforms and Metrics}
\date{} 

\author{Gal~Mishne, Ronen~Talmon, Israel~Cohen, Ronald~R.~Coifman and Yuval~Kluger
\thanks{G. Mishne and R. R. Coifman are with the Department of Mathematics, Yale University, New Haven, CT 06520 USA (e-mail: gal.mishne@yale.edu ; ronald.coifman@math.yale.edu.).
R. Talmon and I. Cohen are with the Viterbi Faculty of Electrical Engineering, Technion - Israel Institute of Technology, Haifa 32000, Israel. (e-mail: ronen@ee.technion.ac.il; icohen@ee.technion.ac.il). 
Y. Kluger is with the Department of Pathology and the Yale Cancer Center, Yale University School of Medicine, New Haven, CT 06511 USA (e-mail: yuval.kluger@yale.edu).
This research was supported by the Israel Science Foundation (grant no. 576/16), and by the United States-Israel Binational Science Foundation and the United States National Science Foundation (grant no. 2015582), and by the National Institutes of Health (grant no. 1R01HG008383-01A1).
}
}

\maketitle

\begin{abstract}
We consider the analysis of high dimensional data given in the form of a matrix with columns consisting of observations and rows consisting of features. 
Often the data is such that the observations do not reside on a regular grid, and the given order of the features is arbitrary and does not convey a notion of locality.
Therefore, traditional transforms and metrics cannot be used for data organization and analysis.
In this paper, our goal is to organize the data by defining an appropriate representation and metric such that they respect the smoothness and structure underlying the data.
We also aim to generalize the joint clustering of observations and features in the case the data does not fall into clear disjoint groups.
For this purpose, we propose multiscale data-driven transforms and metrics based on trees. 
Their construction is implemented in an iterative refinement procedure that exploits the co-dependencies between features and observations.
Beyond the organization of a single dataset, our approach enables us to transfer the organization learned from one dataset to another and to integrate several datasets together. 
We present an application to breast cancer gene expression analysis: learning metrics on the genes to cluster the tumor samples into cancer sub-types and validating the joint organization of both the genes and the samples. 
We demonstrate that using our approach to combine information from multiple gene expression cohorts, acquired by different profiling technologies, improves the clustering of tumor samples.
\end{abstract}
\vspace{-0.1cm}

\section{Introduction}
High-dimensional datasets are typically analyzed as a two-dimensional matrix where, for example, the rows consist of features and the columns consist of observations. 
Signal processing addresses the analysis of such data as residing on a regular grid, such that the rows and columns are given in a particular order, indicating smoothness.
For example, the ordering in time-series data indicates temporal-frequency smoothness, and the order in 2D images indicating spatial smoothness.
Non-Euclidean data that do not reside on a regular grid, but rather on a graph, raise the more general problem of \emph{matrix organization}.
In such datasets, the given ordering of the rows (features) and columns (observations) does not indicate any degree of smoothness.

However, in many applications, for example, analysis of gene expression data, text documents, psychological questionnaires and recommendation systems~\cite{Cheng2000, Tang2001,  Lee2010, Yang2011,  Chi2014,Jiang2004,Bennett2007, Busygin2008, Gavish2012,  Ankenman2014}, there is an underlying structure to both the features and the observations. 
For example, in gene expression subsets of samples (observations) have similar genetic profiles, while subsets of genes (features) have similar expressions across groups of samples.
Thus, as the observations are viewed as high-dimensional vectors of features, one can swap the role of features and observations, and treat the features as high-dimensional vectors of observations. 
This dual analysis reveals meaningful joint structures in the data.

The problem of matrix organization considered here is closely related to biclustering~\cite{Cheng2000,Tang2001,Kluger2003, Lee2010, Yang2011, Tan2014,Chi2014}, where the goal is to identify biclusters: joint subsets of features and observations such that the matrix elements in each subset have similar values.
Matrix organization goes beyond the extraction of joint clusters, yielding a joint reordering of the entire dataset and not just the extraction of partial subsets of observations and features that constitute bi-clusters.
By recovering the smooth joint organization of the features and observations, one can apply signal processing and machine learning methods such as denoising, data completion, clustering and classification, or extract meaningful patterns for exploratory analysis and data visualization.

The application of traditional signal processing transforms to data on graphs is not straightforward, as these transforms rely almost exclusively on convolution with filters of finite support, and thus are based on the assumption that the given ordering of the data conveys smoothness.
The field of graph signal processing adapts classical techniques to signals supported on a graph (or a network), such as filtering and wavelets in the graph domain~\cite{Gavish2010,Singh2010,Hammond2011, Sharpnack2013,Shuman2013, Narang2013, Sakiyama2016, Shuman2016,Tremblay2016,Shahid2016}.
Consider for example signals (observations) acquired from a network of sensors (features). 
The nodes of the graph are the sensors and the edges and their weights are typically dictated by a-priori information such as physical connectivity, geographical proximity, etc.
The samples collected from all sensors at a given time compose a high-dimensional graph signal supported on this network.
The signal observations, acquired over time, are usually processed separately and the connectivity between the observations is not taken into account.

To address this issue, in this paper we propose to analyze the data in a matrix organization setting as represented by \emph{two} graphs: one whose nodes are the observations and the other whose nodes are the features, and our aim is a joint unsupervised organization of these two graphs.
Furthermore, we do not fix the edge weights by relying on a predetermined structure or a-priori information.
Instead, we calculate the edge weights by taking into account the underlying dual structure of the data and the coupling between the observations and the features.
This requires defining two metrics, one between pairs of observations and one between pairs of features.

Such an approach for matrix organization was introduced by Gavish and Coifman~\cite{Gavish2012}, where the organization of the data relies on the construction of a pair of hierarchical partition trees on the observations and on the features.
In previous work~\cite{Mishne2015b}, we extended this methodology to the organization of a rank-3 tensor (or a 3D database), introducing a multiscale averaging filterbank derived from partition trees.

Here we introduce a new formulation of the averaging filterbank as a tree-based linear transform on the data, and propose a new tree-based difference transform.
Together these yield a multiscale representation of both the observations and the features, in analogue to the Gaussian and Laplacian pyramid transforms~\cite{Burt1983}.
Our transforms can be seen as data-driven multiscale filters on graphs, where in contrast to classical signal processing, the support of the filters is non-local and depends on the structure of the data.
From the transforms, we derive a metric in the transform space that incorporates the multiscale structure revealed by the trees~\cite{Leeb2013}. 
The trees and the metrics are incorporated in an iterative bi-organization procedure following~\cite{Gavish2012}.
We demonstrate that beyond the organization of a single dataset, our metric enables us to apply the organization learned from one dataset to another and to integrate several datasets together. 
This is achieved by generalizing our transform to a new multi-tree transform and to a multi-tree metric, which integrate a set of multiple trees on the features. 
Finally, the multi-tree transform inspires a local refinement of the partition trees, improving the bi-organization of the data. 

The remainder of the paper is organized as follows. 
In Section~\ref{sec:problem}, we formulate the problem, present an overview of our solution and review related background. 
In Section~\ref{sec:transform}, we present the new tree-induced transforms and their properties.
In Section~\ref{sec:tree_emd}, we derive the metric in the transform space and propose different extensions of the metric.
We also propose a local refinement of the bi-organization approach.
Section~\ref{sec:results} presents experimental results in the analysis of breast cancer gene expression data.

\subsection{Related Work}
Various methodologies have been proposed for the construction of wavelets on graphs, including Haar wavelets, and wavelets based on spectral clustering and spanning tree decompositions{~\cite{Gavish2010, Singh2010, Hammond2011,Ram2011, Sharpnack2013,Kondor2014}}.  Our work deviates from this path and presents an iterative construction of data-driven tree-based transforms. 
In contrast to previous multiscale representations of a single graph, our approach takes into account the co-dependencies between observations and features by incorporating two graph structures.
Our motivation for the proposed transforms is the tree-based Earth Mover's Distance (EMD) proposed in~\cite{Leeb2013}, which introduces a coupling between observations and features, enabling an iterative procedure that updates the trees and metrics in each iteration. 
The averaging transform, in addition to being equipped with this metric, is also easier to compute than a wavelet basis as it does not require an orthogonalization procedure. 
In addition, given a partition tree, the averaging and difference transforms are unique, whereas the tree-based wavelet transform~\cite{Gavish2010} on a non-binary tree is not uniquely defined.  
Finally, since the averaging transform is over-complete such that each filter corresponds to a single folder in the tree, it is simple to design weights on the transform coefficients based on the properties of the individual folders. 

Filterbanks and multiscale transforms on trees and graphs have been proposed in~\cite{Narang2013, Sakiyama2016, Shuman2016, Tremblay2016}, yet differ from our approach in several aspects. 
While filterbanks construct a multi-scale representation by using downsampling operators on the data~\cite{Shuman2016, Narang2013}, the multiscale nature of our transform arises from partitioning of the data via the tree. 
In that, it is most similar to~\cite{Tremblay2016}, where the graph is decomposed into subgraphs by partitioning. 
However, all these filterbanks on graphs employ the eigen-decomposition of the graph Laplacian to define either global filters on the full graph or local filters on disjoint subgraphs. 
Our approach, conversely, employs the eigen-decomposition of the graph Laplacian to construct the partition tree, but the transforms (filters) are defined by the structure of the tree and not explicitly derived from the Laplacian.
In addition, we do not treat the structure of the graph as fixed, but rather iteratively update the Laplacian based on the tree transform.
Finally, while graph signal processing typically addresses one dimension of the data (features or observations), our approach addresses the construction of transforms on both the observations and features of a dataset, and relies on the coupling between the two to derive the transforms.

This work is also related to the matrix factorization proposed by Shahid et al.~\cite{Shahid2016}, where the graph Laplacians of both the features and the observation regularize the decomposition of a dataset into a low-rank matrix and a sparse matrix representing noise.
Then the observations are clustered using k-means on the low-dimensional principal components of the smooth low-rank matrix. 
Our work differs in that we preform an iterative \emph{non-linear} embedding of the observations and features, not jointly, but alternating between the two while updating the graph Laplacian of each in turn.
In addition, we provide a \emph{multiscale} clustering of the data. 

\section{Bi-organization}
\label{sec:problem}
\subsection{Problem Formulation}
Let $\mathbf{Z}$ be a high-dimensional dataset and let us denote its set of $n_\X$ features by $\X$ and denote its set of $n_\Y$ observations by $\Y$.
For example, in gene expression data, $\X$ consists of the genes and $\Y$ consists of individual samples.
The element $Z(x,y)$ is the expression of gene $x \in \X$ in sample $y \in \Y$.
The given ordering of the dataset is arbitrary such that adjacent features and adjacent observations in the dataset are likely dissimilar. 
We assume there exists a reordering of the features and a reordering of the observations such that $\mathbf{Z}$ is smooth.
\begin{definition}
A matrix $\mathbf{Z}$ is smooth if it satisfies the mixed H\"{o}lder condition~\cite{Gavish2012}, such that $\forall x,x' \in \X$ and $\forall y,y' \in \Y$, and for a pair of non-trivial metrics $\rho_\X$ on $\X$ and $\rho_\Y$ on $\Y$ and constants $C>0$ and $0<\alpha \leq 1$:
\begin{multline}
\label{eq:mixedH}
\vert \mathbf{Z}(x,y) - \mathbf{Z}(x,y')-\mathbf{Z}(x',y)+\mathbf{Z}(x',y') \vert \\
 \leq C \rho_\X(x,x')^\alpha \rho_\Y(y,y')^\alpha.
\end{multline}
\end{definition}
Note that we do not impose smoothness as an explicit constraint; instead it manifests itself implicitly in our data-driven approach.

Although the given ordering of the dataset is not smooth, the organization of the observations and the features by partition trees following~\cite{Gavish2012} constructs both local and global neighborhoods of each feature and of each observation.
Thus, the structure of the tree organizes the data in a hierarchy of nested clusters in which the data is smooth.
Our aim is to define a transform on the features and on the observations that conveys the hierarchy of the trees, thus recovering the smoothness of the data. 
We define a new metric in the transform space that incorporates the hierarchical clustering of the data via the trees.
The notations in this paper follow these conventions: matrices are denoted by bold uppercase and sets are denoted by uppercase calligraphic.

\subsection{Method Overview}
The construction of the tree, which relies on a metric, and the calculation of the metric, which is derived from a tree, lead to an iterative bi-organization algorithm~\cite{Gavish2012}.
Each iteration updates the pair of trees and metrics on the observations and features as follows.
First, an initial partition tree on the features, denoted $\T_\X$, is calculated based on an initial pairwise affinity between features.
This initial affinity is application dependent. 
Based on a coarse-to-fine decomposition of the features implied by the partition tree on the features, we define a new metric between pairs of observations: $d_{\T_\X}(y,y')$.
The metric is then used to construct a new partition tree on the observations $\T_\Y$.
Thus, the construction of the tree on the observations $\T_\Y$ is based on a metric induced by the tree on the features $\T_\X$.
The new tree on the observations $\T_\Y$ then defines a new metric between pairs of features $d_{\T_\Y}(x,x')$.
Using this metric, a new partition tree is constructed on the features $\T_\X$, and a new iteration begins.
Thus, this approach exploits the strong coupling between the features and the observations.
This enables an iterative procedure in which the pair of trees are refined from iteration to iteration, providing in turn a more accurate metric on the features and on the observations.
We will show that the resulting tree-based transform and corresponding metric enable a multiscale analysis of the dataset, reordering of the observations and features, and detection of meaningful joint clusters in the data.

\subsection{Partition trees}
\label{sec:trees}
Given a dataset $\mybf{Z}$, we construct a hierarchical partitioning of the observations and features defined by a pair of trees.
Without loss of generality, we define the partition trees in this section with respect to the features, and introduce relevant notation.

Let $\mathcal{T}_\X$ be a partition tree on the features.
The partition tree is composed of $L+1$ levels, where a partition $\mathcal{P}_l$ is defined for each level $0 \leq l \leq L$.   
The partition $\mathcal{P}_l = \{\mathcal{I}_{l,1},...,\mathcal{I}_{l,n(l)}\}$ at level $l$ consists of $n(l)$ mutually disjoint non-empty subsets of indices in $\{1,...,n_X\}$, termed folders and denoted by $\mathcal{I}_{l,i}$, $i\in\{1,...,n(l)\}$.
Note that we define the folders on the indices of the set and not on the features themselves.

The partition tree $\mathcal{T}_\X$ has the following properties:
\begin{itemize}
\item The finest partition ($l = 0$) is composed of $n(0) = n_\X$ singleton folders, termed the ``leaves'', where $\mathcal{I}_{0,i} = \{i\}$.
\item The coarsest partition ($l= L$) is composed of a single folder, $\mathcal{P}_L=\mathcal{I}_{L,1} = \{1,...,n_\X\}$, termed the ``root''. 
\item The partitions are nested such that if $\mathcal{I} \in \mathcal{P}_l$, then $\mathcal{I} \subseteq \mathcal{J}$ for some $\mathcal{J} \in \mathcal{P}_{l+1}$, i.e., each folder at level $l-1$ is a subset of a folder from level $l$.
\end{itemize}
The partition tree is the set of all folders at all levels $\mathcal{T}_\X = \{\mathcal{I}_{l,i} \;\vert\; 0 \leq l \leq L,\; 1 \leq i \leq n(l)\}$, and the number of all folders in the tree is denoted by $N_\X = \vert \mathcal{T}_\X \vert$.
The size, or cardinality, of a folder $\mathcal{I}$, i.e. the number of indices in that folder, is denoted by $\vert \mathcal{I} \vert$.
In the remainder of the paper, for compactness, we drop the subscript $l$ denoting the level of a folder, and denote a single folder by either $\I$ or $\I_i$, such that $i \in \{1,...,N_\X\}$ is an index over all folders in the tree.

Given a dataset, there are many methods to construct a partition tree, including deterministic, random, agglomerative (bottom-up) and divisive (top-down)~\cite{Gavish2010,Chi2014,Breiman2001}.
For example, in a bottom-up approach, we begin at the lowest level of the tree and cluster the features into small folders. 
These folders are then clustered into larger folders at higher levels of the tree, until all folders are merged together at the root.

Some approaches take into account the geometric structure and multiscale nature of the data by incorporating affinity matrices defined on the data, and manifold embeddings~\cite{Gavish2010,Ankenman2014}.
Ankenman~\cite{Ankenman2014} proposed ``flexible trees'', whose construction requires an affinity kernel defined on the data, and is based on a low-dimensional diffusion embedding of the data~\cite{Coifman2006}. 
Given a metric between features $d(x,x')$, a local pairwise affinity kernel $k(x,x')=\exp\{-d(x,x')/\sigma^2\}$ is integrated into a global representation on the data via a manifold embedding representation $\Psi$, which minimizes
\begin{equation}
 \min \sum_{x,x'} k(x,x') \Vert \Psi(x) - \Psi(x') \Vert^2_2.
\end{equation}
The clustering of the folders in the flexible tree algorithm is based on the Euclidean distance between the embedding $\Psi$  of the features, which integrates the original metric $d(x,x')$.
Thus, the construction of the tree does not rely directly on the high-dimensional features but on the low-dimensional geometric representation underlying the data (see~\cite{Ankenman2014} for a detailed description).
The quality of this representation, and therefore, of the constructed tree depends on the metric $d(x,x')$.
In our approach, we propose to use the metric induced by the tree on the observations $d(x,x')=d_{\T_\Y}(x,x')$.
This introduces a coupling between the observations and the features, as the tree construction of one depends on the tree of the other.
Since our approach is based on an iterative procedure, the tree construction is refined from iteration to iteration, as both the tree and the metric on the features are updated based on the organization of the observations, and vice versa.
This also updates the affinity kernel between observations and the affinity kernel between features, therefore updating the dual graph structure of the dataset.

Note that while we apply flexible trees in our experimental results, the bi-organization approach is modular and different tree construction algorithms can be applied, as in~\cite{Coifman2011,Gavish2012}. 
While the definition of the proposed transforms and metrics does not depend on properties of the flexible trees algorithm, the resulting bi-organization does depend on the tree construction.
Spin-cycling (averaging results over multiple trees) as in~\cite{Ankenman2014} can be applied to stabilize the results. 
Instead, we propose an iterative refinement procedure that makes the algorithm less dependent on the initial tree constructions.
Convergence guarantees to smooth results from a family of appropriate initial trees are lacking. This will be the subject of future work.

\section{Tree transforms}
\label{sec:transform}
Given partition trees $\mathcal{T}_\X$ and $\mathcal{T}_\Y$, defined on the features and observations, respectively, we propose several transforms induced by the partition trees, which are defined by a linear transform matrix and generalizes the method proposed in~\cite{Ankenman2014}.
In the following we focus on the feature set $\X$, but the same definitions and constructions apply to the observation set $\Y$.
Note that while the proposed transforms are linear, the support of the transform elements is derived in a non-linear manner as it depends on the tree construction. 

The proposed transforms project the data onto a high dimensional space whose dimensionality is equal to the number of folders in the tree, denoted by $N_\X$, i.e. the transform maps $\mathbf{T} : \mathbb{R}^{n_\X} \rightarrow \mathbb{R}^{N_\X}$.
Each transform is represented as a matrix of size ${N_\X} \times n_\X$, where $n_\X$ is the number of features.
We denote the row indices of the transform matrices by $i,j \in \{1,2,...,N_\X\}$ indicating the unique index of the folder in $\mathcal{T}_\X$.
We denote the column indices of the transform matrices by $x,x' \in \X$ ($y,y' \in \Y$), which are the indices of the features (observations) in the data.
We define $\mathds{1}_\mathcal{I}$ to be the indicator function on the features $x\in\{1,...,n_\X\}$ belonging to folder $\mathcal{I} \in \mathcal{T}_\X$.
Tree transforms obtained from $\T_\X$ are applied to the dataset as $\hat{\mybf{Z}}_\X=\mathbf{T}_\X \mathbf{Z}$ and tree transforms obtained from $\T_\Y$ are applied to the dataset as $\hat{\mybf{Z}}_\Y=\mathbf{Z} \mathbf{T}_\Y^T$ .
We begin with transforms induced by a tree in a single dimension (features or observations) analogously to a typical one-dimensional linear transform.
We then extend these transforms to joint-tree transforms induced by a pair of trees $\{\T_\X,\T_\Y\}$ on the observations and the features, analogously to a two-dimensional linear transform.
Finally, we propose multi-tree transforms in the case that we have more than one tree in a single dimension, for example we have constructed a set of trees $\{\T_\X\}$ on the features $\X$, each constructed from a different dataset consisting of different observations with the same features.

\subsection{Averaging transform}
Let $\mathbf{S}$ be an ${N_\X \times n_\X}$ matrix representing the structure of a given tree $\mathcal{T}_\X$, by having each row $i$ of the matrix be the indicator function of the corresponding folder $\mathcal{I}_i \in \mathcal{T}_\X$:
\begin{equation}
\mathbf{S}[i,x] = \mathds{1}_{\mathcal{I}_i}(x) = \left\{
  \begin{array}{lr}
    1, & x\in \mathcal{I}_i\\
    0, & \textrm{otherwise}
  \end{array}
\right.
\end{equation}
Applying $\mathbf{S}$ to an observation vector $y\in \mathbb{R}^{n_\X}$ yields a vector of length $N_\X$ where each element $i\in \{1,...,N_\X\}$ is the sum of the elements $y(x)$ for $x\in\I_i$:
\begin{equation}
(\mathbf{S}y)[i]=\sum_{x\in \X} y(x)\mathds{1}_{\mathcal{I}_i}(x) = \sum_{x\in \I_i} y(x)
\end{equation}

The sum of each row of $\mathbf{S}$ is the size of its corresponding folder: $\sum_x \mathbf{S}[i,x] = \vert {\mathcal{I}_i} \vert.$
The sum of each column is the number of levels in $\T_\X$: $\sum_i \mathbf{S}[i,x] = L+1$, since the folders are disjoint at each level such that each feature belongs only to a single folder at each level.

From $\mybf{S}$ we derive the averaging transform denoted by $\mathbf{M}$.
Let $\mathbf{D}\in \mathbb{R}^{N_\X \times N_\X}$ be a diagonal matrix whose elements are the cardinality of each folder: $\mathbf{D}[i,i] = \vert {\mathcal{I}_i} \vert$.
We calculate $\mathbf{M}\in \mathbb{R}^{N_\X \times n_\X}$ by normalizing the rows of $\mathbf{S}$, so the sum of each row is 1:
\begin{equation}
\mathbf{M}  = \mathbf{D}^{-1} \mathbf{S}.
\end{equation}
Thus, the rows $i$ of $\mathbf{M}$ are uniformly weighted indicators on the indices of $\X$ for each folder $\mathcal{I}_i$:
\begin{equation}
\mathbf{M}[i,x]  = \frac{1}{\vert \mathcal{I}_i \vert}\mathds{1}_{\mathcal{I}_i}(x)
 = \left\{
  \begin{array}{lr}
    \frac{1}{\vert \mathcal{I}_i \vert}, & x\in \mathcal{I}_i\\
    0, & \textrm{o.w.}
  \end{array}
\right.
\end{equation}
Note that the matrix $\mybf{S}$ and the averaging transform $\mybf{M}$ share the same structure, i.e. they differ only in the value of the their non-zero elements.

Alternatively if we denote by $m(y,\I)$ the average value of $y(x)$ in folder $\I$:
\begin{equation}
\label{eq:meanI}
m(y,\I) = \frac{1}{\cardI{}}\sum_{x\in\I} y(x),
\end{equation}
then applying the averaging transform $\mathbf{M}$ to $y$ yields a vector $\hat{y}$ of length $N_\X$ such that each element $i$ is the average value of $y$ in folder $\I_i$ ~(\ref{eq:meanI}):
\begin{equation}
\hat{y}[i]=(\mathbf{M}y)[i]= m(y,\I_i), \;\; 1\leq i \leq N_\X.
\end{equation}

\begin{figure}[t]
\centering{\includegraphics[width=0.90\linewidth]{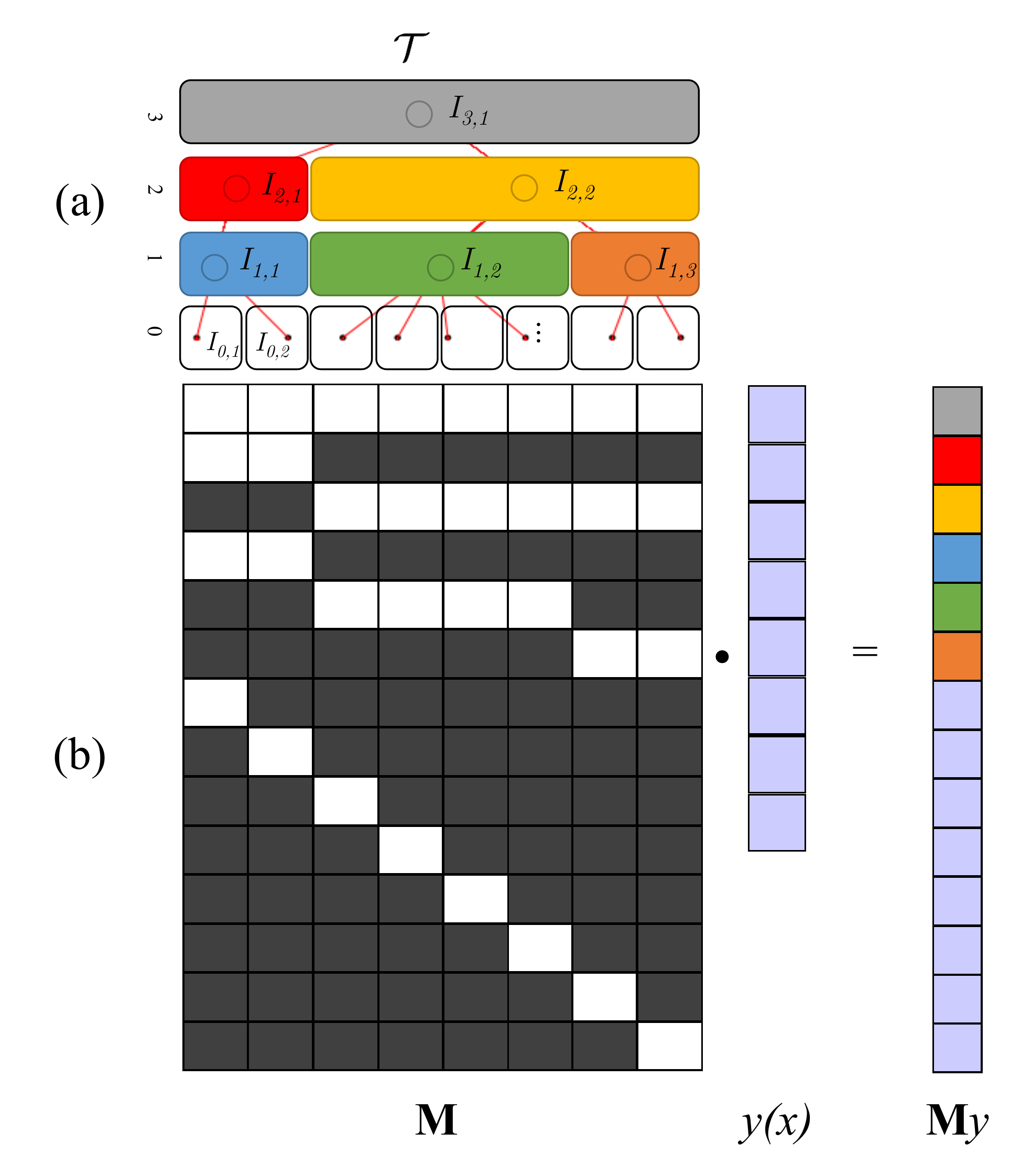}}
\caption{(a) Partition tree $\T$. (b) Averaging transform matrix $\mybf{M}$ induced by the tree and applied to column vector $y(x)$. The color of the elements in the output correspond to the color of the folders in the tree. }
\label{fig:transform}
\end{figure}
The averaging transform reinterprets each folder in the tree as applying a uniform averaging filter, whose support depends on the size of the folder.
Applying the feature-based transform $\mathbf{M}_\X$ to the dataset $\mathbf{Z}$ yields $\hat{\mathbf{Z}}_\X= \mybf{M}_\X\mybf{Z} \in \mathbb{R}^{N_\X \times n_\Y}$, a data-driven multi-scale representation of the data. 
As opposed to a multiscale representation defined on a regular grid, here the representation at each level is obtained via non-local averaging of the coefficients from the level below.
The finest level of the representation is the data itself, which is then averaged in increasing degree of coarseness and in a non-local manner according to the clusters defined by the hierarchy in the partition tree.
The columns of $\hat{\mathbf{Z}}_\X$ are the multiscale representation $\hat{y}$ of each observation $y$.
The rows of $\hat{\mathbf{Z}}_\X$ are the centroids of the folders $\I \in \T_\X$ and can be seen as multiscale \emph{meta-features} of length $n_\Y$:
\begin{equation}
C_i(y) = \sum_x \mybf{M}[i,x] \mybf{Z}[x,y], \;\; 1\leq y \leq n_\Y.
\end{equation}
In a similar fashion denote by $\hat{\mathbf{Z}}_\Y = \mybf{Z} \mybf{M}_\Y^T$ the application of the observation-based transform to the entire dataset.
For additional properties of $\mybf{S}$ and $\mybf{M}$ see~\cite{MishneThesis}.

In Fig.~\ref{fig:transform}, we display an illustration of a partition tree and the resulting averaging transform.
Fig.~\ref{fig:transform}(a) is a partition tree $\T_\X$ constructed on $\X$ where $n_\X=8$.
Fig.~\ref{fig:transform}(b) is the averaging transform $\mathbf{M}$ corresponding to the partition tree $\T_\X$.
For visualization purposes we construct $\mybf{M}$ as having columns whose order correspond to the leaves of the tree $\T_\X$ (level 0).
This reordering also needs to be applied to the data vectors $y$, and is essentially one of the aims of our approach.
The lower part of the transform is just the identity matrix, as it corresponds to the leaves of the tree.
The number of rows in the transform matrix is $N_\X = \vert \T \vert = 14$, as the number of folders in the tree.
The transform is applied to a (reordered) column $y \in \mathbb{R}^8$, yielding the coefficient vector $\hat{y}=\mathbf{M}y \in \mathbb{R}^{14}$.
The coefficients are colored according to the corresponding folders in the tree.

To further demonstrate and visualize the transform, we apply the averaging transform to an image in Fig.~\ref{fig:lena1d}.
We treat a grayscale image as a high-dimensional dataset where $\X$ is the set of rows and $\Y$ is the set of columns.
We calculate a partition tree $\T_\Y$ on the columns.
We then calculate the averaging transform and apply it to the image yielding $\hat{\mathbf{Z}}_\Y = \mathbf{Z}\mathbf{M}_\Y^T$. 
The result is presented in Fig.~\ref{fig:lena1d}(a).
Each row $x$ has now been extended to a higher dimension $N_\Y$, where we separate the levels of the tree with colored borders lines for visualization purposes. 
Each of the columns $\hat{\mathbf{Z}}_\Y$ is the centroid of folder $\I$ in the tree.
The right-most sub-matrix is the original image and as we move left we have coarser and coarser scales.
The averaging is non-local and the folder sizes vary, respecting the structure of the data. 
Thus on the second level of the tree, the building on the right is more densely compressed compared to the building on the left. 

\begin{figure}[t]
\centering{\includegraphics[width=0.90\linewidth]{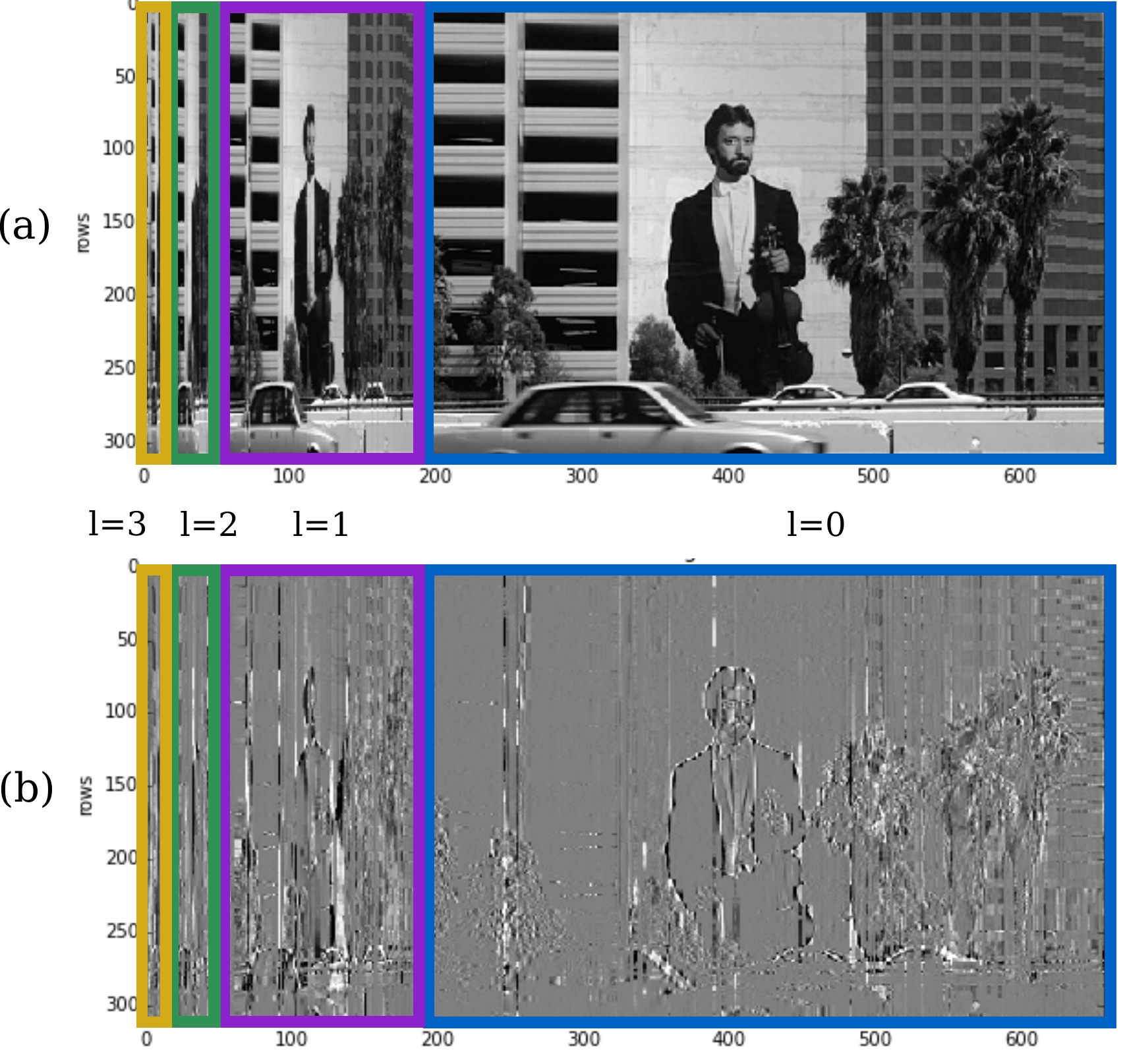}}
\caption{Application of the averaging transform (a) and the difference transform (b) to an image. The color of the border represents the level of the tree. The non-local nature of the transforms and the varying support is apparent, for example, in the building on the right. In the fine-scale resolution the building has 7 windows in the horizontal direction, which have been compressed into 5 windows on the next level.}\label{fig:lena1d}
\end{figure}

\subsection{Difference transform}
The goal of our approach is to organize the data in nested folders in which the features and the observations are smooth.
Thus, it is of value to determine how smooth is the hierarchical structure of the tree, i.e. does the merging of folders on one level into a single folder on the next level preserve smoothness.
Let $\mathbf{\Delta}$ be an ${N_\X \times n_\X}$ matrix, termed the multiscale difference transform. 
This transform yields the difference between $\hat{y}[i]$ and $\hat{y}[j]$ where $j$ is the index of the immediate parent of folder $i$.

The matrix $\mathbf{\Delta}$ is obtained from the averaging matrix $\mathbf{M}$ as:
\begin{equation}
\mathbf{\Delta}[i,x] = \mybf{M}[i,x]-\mybf{M}[j,x], \;\; \I_{l,i} \subset \I_{l+1,j}.
\end{equation}
Applying $\mathbf{\Delta}$ to observation $y$ yields a vector of length $N_\X$ whose element $i$ is the difference between the average value of $y$ in folder $\mathcal{I}_{l,i}$ and the average value in its immediate parent folder $\mathcal{I}_{l+1,j}$:
\begin{equation}
(\mathbf{\Delta}y)[i]= \left\{
  \begin{array}{lr}
	m(y,\I_{L,1}), & \I_i = \I_{L,1}\\
	m(y,\I_{l,i})-m(y,\I_{l+1,j}), & \I_{l,i} \subset\I_{l+1,j},
  \end{array}
\right.
	\end{equation}
where for the root folder, we define $(\mathbf{\Delta}y)[i]$ to be the average over all features.
This choice leads to the definition of an inverse transform below.
Thus, the rows $i$ of $\mathbf{\Delta}$ are given by:
\begin{equation}
\mathbf{\Delta}[i,x] = \left\{
  \begin{array}{lr}
	\frac{1}{\vert \mathcal{I}_i \vert}, & \I_i = \I_{L,1} \\
    \frac{1}{\cardI{{l,i}}} - \frac{1}{\cardI{{l+1,j}}}, & x\in \I_{l,i} \subset\I_{l+1,j} \\
		- \frac{1}{\cardI{{l+1,j}}}, & x\notin \I_{l,i} \subset\I_{l+1,j}, x \in \I_{l+1,j} \\
    0, & x \notin \I_{l,i}, x \notin \I_{l+1,j}
  \end{array}
\right.
\end{equation}
and the sum of the rows of $\mybf{\Delta}$:
\begin{equation}
\mathbf{\Delta}[i,x] = \left\{
  \begin{array}{lr}
	1, & \I_i = \I_{L,1} \\
  0, & \textrm{otherwise}
  \end{array}
\right.
\end{equation}
The difference transform can be seen as revealing ``edges" in the data, however these edges are non-local.
Since the tree groups features together based on their similarity and not based on their adjacency, the difference between folders is not restricted to the given ordering of the features.
This demonstrated in Fig.~\ref{fig:lena1d}(b) where the difference transform of the column tree has been applied to the 2D image as $\mathbf{Z}\mathbf{\Delta}_\Y^T$. 

\begin{theorem}
The data can be recovered from the difference transform by:
\begin{equation}
y=\mybf{S}^T(\mybf{\Delta}y)
\end{equation}
\end{theorem}
\begin{proof}
An element $(\mybf{S}^T\mybf{\Delta}y)[x]$ is given by
\begin{equation}
\begin{split}
& \sum_{\mathclap{\substack{\I_{l,i}\in \T_\X \\  \I_{l,i} \subset \I_{l+1,j} \\ l < L}}} \mathds{1}_{\mathcal{I}_i}(x) \left( m(y,\I_{l,i})-m(y,\I_{l+1,j}) \right) + m(y,\I_{L,1}) = \\
& = \sum_{\mathclap{\substack{\I_{l,i}\in \T_\X \\  0 \leq l \leq L}}}  \mathds{1}_{\mathcal{I}_i}(x) m(y,\I_{l,i}) - \sum_{\mathclap{\substack{\I_{l,i}\in \T_\X \\  1 \leq l \leq L}}}  \mathds{1}_{\mathcal{I}_i}(x) m(y,\I_{l,i}) = \\
& = \sum_{i=1}^{n(0)} \mathds{1}_{\mathcal{I}_i}(x) m(y,\I_{0,i}) = y(x)  \;\; \hfill \square
\end{split}
\end{equation}
\end{proof}
The first equality is due to the folders on each level being disjoint such that if $x\in \I_{l,i}$ and $\I_{l,i} \subset \I_{l+1,j}$ then $x\in \I_{l+1,j}$, and $\I_{l+1,j}$ is the only folder containing $x$ on level $l+1$.
This enables us to process the data in the tree-based transform domain and then reconstruct by:
\begin{equation}
\hat{y}=\mybf{S}^Tf(\mybf{\Delta}y),
\end{equation}
where $f : \mathbb{R}^N_\X \rightarrow \mathbb{R}^N_\X$ is a function in the domain of the tree folders.
For example, we can threshold coefficients based on their energy or the size of their corresponding folder.
This scheme can be applied to denoising and compression of graphs or matrix completion~\cite{Narang2013, Sakiyama2016, Shuman2016, Tremblay2016}, however this is beyond the scope of this paper and will be explored in future work.

Note that the difference transform differs from the tree-based Haar-like basis introduced in~\cite{Gavish2010}.
The Haar-like basis is an orthonormal basis spanned by $n_\X$ vectors derived from the tree by an orthogonalization procedure. 
The difference transform is overcomplete and spanned by $N_\X$ vectors, whose construction does not require an orthogonalization procedure, making it simpler to compute. 
Also, as each vector corresponds to a single folder, it enables us to define a measure of the homogeneity of a specific folder compared to its parent.

\vspace{-0.2cm}

\subsection{Joint-tree transforms}
Given the matrix $\mybf{Z}$ on $\X \times \Y$, and the respective partition trees $\mathcal{T}_\X$ and $\mathcal{T}_\Y$, we define joint-tree transforms that operate on the features and observations of $\mybf{Z}$ simultaneously. 
This is analogous to typical 2D transforms.
The joint-tree averaging transform is applied as 
\begin{equation}
\hat{\mybf{Z}}_{\X,\Y} = \mybf{M}_\X \mybf{Z}\mybf{M}_\Y^T.
\end{equation} 
The resulting matrix of size $N_\X \times N_\Y$ provides a multiscale representation of the data matrix, admitting a block-like structure corresponding to the folders in both trees.
On the finest level we have $\mybf{Z}$ and then on coarser and coarser scales we have smoothed versions of $\mybf{Z}$, where the averaging is performed under the joint folders at each level. 
The coarsest level is of size $1\times1$ corresponding to the joint root folder.
This matrix is analogous to a 2D Gaussian pyramid representation of the data, popular in image processing~\cite{Burt1983}. 
However, as opposed to the 2D Gaussian pyramid in which each level is a reduction of both dimensions, applying our transform yields all combinations of fine and coarse scales in both dimensions.
The joint-tree averaging transform yields a result similar to the directional pyramids introduced in~\cite{Zontak2013}, however the ``blur" and ``sub-sample" operations in our case are data-driven and non-local.

The joint-tree \emph{difference} transform is applied as $\mybf{\Delta}_\X \mybf{Z}\mybf{\Delta}_\Y^T$. 
This matrix is analogous to a 2D Laplacian pyramid representation of the data, revealing ``edges" in the data. 
 As in applying a 1D transform, the data can be recovered from the joint-tree difference transform as
$
 \mybf{Z} = \mybf{S}_\X^T\mybf{\Delta}_\X\mybf{Z}\mybf{\Delta}_\Y^T\mybf{S}_\Y.
$
 
 Figure~\ref{fig:building2d} presents applying the joint-tree averaging transform and joint-tree difference transform to the 2D image.
 Within the red border we display ``zooming in'' on level $l \geq 1$ in both trees $\T_\X$ and $\T_\Y$.
 \begin{figure}[t]
 \centering{\includegraphics[width=0.99\linewidth]{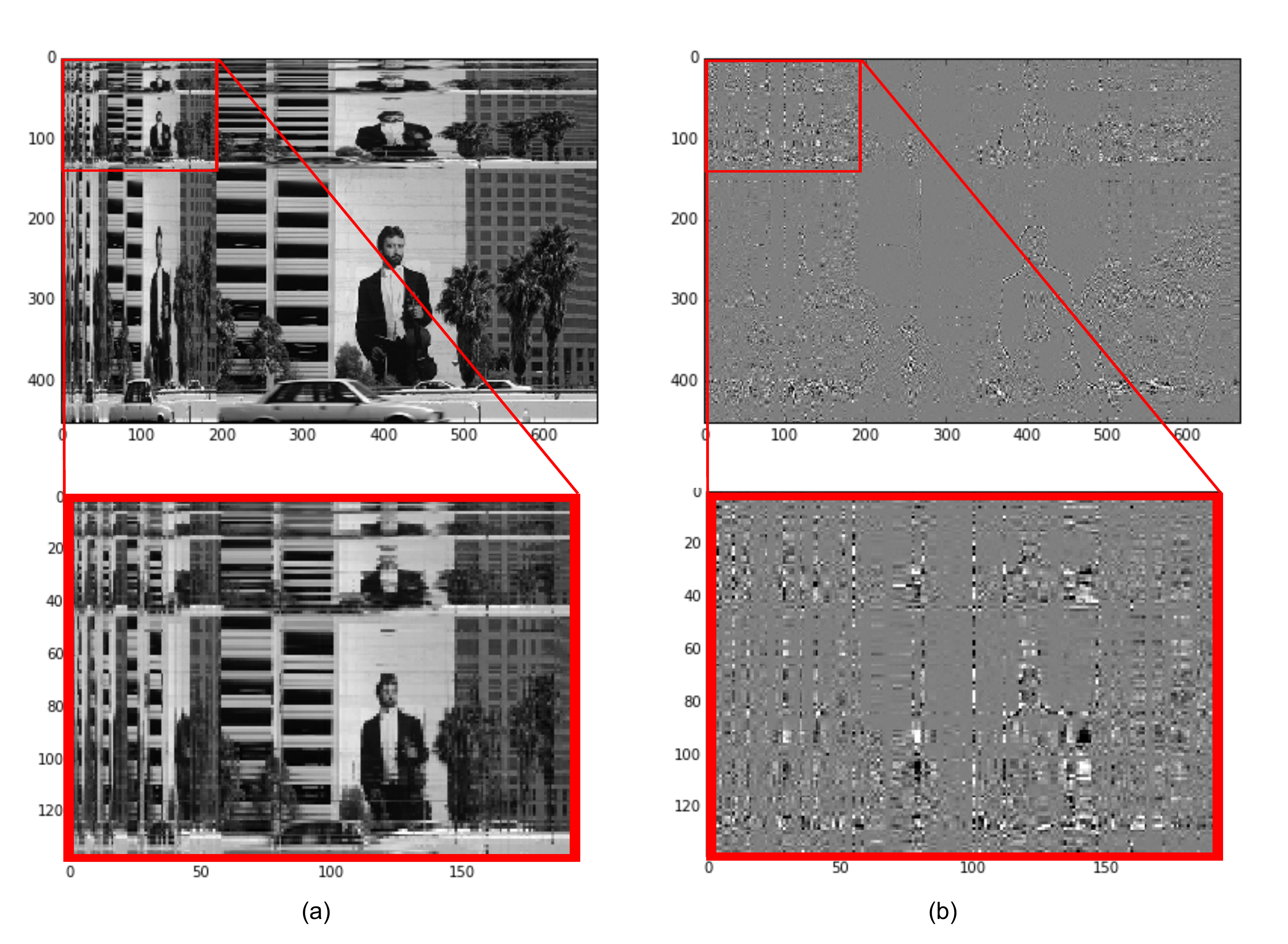}}
 \caption{(a) Joint-tree averaging transform applied to image. (b) Joint-tree difference transform applied to image.}
 \label{fig:building2d}
 \end{figure}
 
\subsection{Multi-tree transforms}
\label{sec:multitree}
At each level of the partition tree, the folders are grouped into disjoint sets. 
A limitation of using partition trees, therefore, is that each folder is connected to a single ``parent".
However, it can be beneficial to enable a folder on one level to participate in several folders at the level above, such that folders overlap, as in~\cite{Ahn2010}.
We propose an approach that enables overlapping folders in the bi-organization framework by constructing more than one tree on the features ${\X}$, and we extend the single tree transforms to multi-tree transforms.
This generalizes the partition tree such that each folder can be connected to more than one folder in the above level, i.e. this is no longer a tree because it is now cyclic but still a bipartite graph. 
Note that in contrast to the joint-tree transform, which incorporates a joint pair of trees over both the features \emph{and} the observations, here we are referring to a set of trees defined for only the features, or only the observations.

Given a set of $n_\T$ different partition trees on $\X$, denoted  $\{\T_t\}_{t=1}^{n_\T}$, we construct the multi-tree averaging transform.
Let $\Mall_\X$ be an $\widetilde{N}_\X\times n_\X$ matrix, constructed by concatenation of the averaging transform matrices $\mybf{M}_\T$ induced by each of the trees $\{\T_t\}_{t=1}^{n_\T}$.
The number of rows in the multi-tree transform matrix is denoted by $\widetilde{N}_\X$ and equal to the number of folders in all of the trees $\sum_t \vert \T_t \vert$. 
Yet since all trees $\{\T_t\}_{t=1}^{n_\T}$ contain the same root and leaves folders, we remove the multiple appearance of the rows corresponding to these folders and include them only once (then $\widetilde{N}_\X =\sum_t \vert \T_t \vert - (n_\T-1)(1+n_\X)$).
Thus, the matrix of the multi-tree averaging transform now represents a decomposition via a single root, a single set of leaves and many multiscale folders that are no longer disjoint.
This implies that instead of considering multiple ``independent'' trees, we have a single hierarchical graph where at each level we do not have disjoint folders, as in a tree, but instead \emph{overlapping} folders. 
In Sec.~\ref{sec:multi_metric}, we derive from these transforms a new multi-tree metric.
For additional properties of the multi-tree transform see~\cite{MishneThesis}.

Ram, Elad and Cohen~\cite{Ram2011} also proposed a ``generalized tree transform" where folders are connected to multiple parents in the level above, however their work differs in two aspects. First, their proposed tree construction is a binary tree, whereas ours admits general tree constructions. Second, their transform relies on classic pre-determined wavelet filters such that the support of the filter is fixed across the dataset. Our formulation on the other hand introduces data-driven filters whose support is determined by the size of the folder, which can vary across the tree.
The Multiresolution Matrix Factorization (MMF)~\cite{Kondor2014} also yields a wavelet basis on graphs.
MMF uncovers a hierarchical organization of the graph that permits overlapping clusters, by decomposition of a graph Laplacian matrix via a sequence of sparse orthogonal matrices.
However, our transform is derived from a set of multiple hierarchical trees, whereas their hierarchical structure is derived from the wavelet transform.

The field of community detection also addresses finding overlapping clusters in graphs~\cite{Xie2013}.
Ahn, Bagrow and Lehmann~\cite{Ahn2010} construct multiscale overlapping clusters on graphs by performing hierarchical clustering with a similarity between \emph{edges} of a graph, instead of its nodes. 
Their approach focuses on the explicit construction of the hierarchy of the overlapping clusters, whereas our focus is on employing a transform and a metric derived from such a multiscale overlapping organization of the features.
In contrast to clustering, our approach allows for the organization and analysis of the observations.

\section{Tree-based metric}
\label{sec:tree_emd}
The success of the data organization and the refinement of the partition trees depends on the metric used to construct the trees. 
We assume that a good organization of the data recovers smooth joint clusters of observations and features.
Therefore, a metric for comparing pairs of observations should not only compare their values for individual features (as in the Euclidean distance), but also across clusters of features, which are expected to have similar values.
Thus, we present a metric $d_\T$ in the multiscale representation yielded by the tree transforms.
Using this metric, the construction of the tree on the features takes into account the structure of the underlying graph on the observations as represented by its partition tree. 
The partition tree on the observations in turn relies on the graph structure of the features.
In each iteration a new tree is calculated based on the metric from the previous iteration, and then a new metric is calculated based on the new tree.
This can be seen as updating the dual graph structure of the data in each iteration.
The iterative bi-organization algorithm is presented in Alg.~\ref{alg:biorg}.

\subsection{Tree-based EMD}
Coifman and Leeb~\cite{Leeb2013} define a tree-based metric approximating the EMD in the setting of hierarchical partition trees. 
Given a 2D matrix $\mybf{Z}$, equipped with a partition tree on the features $\mathcal{T}_\X$, consider two observations $y,y' \in \Y$.
The tree-based metric between the observations is defined as
\begin{equation}
\label{eq:emd}
d_{\mathcal{T}_\X}(y,y') = \sum_{\mathcal{I} \in \mathcal{T}_\X} \left(\frac{\vert \I \vert}{n_\X}\right)^\beta \vert m(y - y', \mathcal{I})\vert,
\end{equation}
where $\beta$ is a parameter that weights the folders in the tree based on their size.
Following our formulation of the trees inducing linear transforms, this tree-based metric can be seen as a weighted $l_1$ distance in the space of the averaging transform.

\begin{theorem}
\label{thm:emd}
\cite[Theorem 4.1]{Mishne2015b}
Given a partition tree on the features $\mathcal{T}_\X$, define the $N_\X \times N_\X$ diagonal weight matrix $\mathbf{W}[i,i] =\left(\frac{\vert \I_i \vert}{n_\X}\right)^\beta$.
 Then the tree metric~(\ref{eq:emd}) between two observations $y,y' \in \mathbf{R}^{n_\X}$ is equivalent to the weighted $l_1$ distance between the averaging transform coefficients:
\begin{equation}
\label{eq:emd_transform}
d_{\mathcal{T}_\X}(y,y') = \Vert \mathbf{W} (\hat{y}-\hat{y}') \Vert_1.
\end{equation}
\end{theorem}

\begin{proof}
An element of the vector $\mathbf{W} (\hat{y}-\hat{y}')$ is
 \begin{equation}
 \begin{split}
(\mathbf{W} \mathbf{M} (y-y'))[i] & = \sum_j \mathbf{W}[i,j] (\mathbf{M} (y-y'))[j] \\
& =\mathbf{W}[i,i] (\mathbf{M} (y-y'))[i] \\
& = \left(\frac{\vert \I_i \vert}{n_\X}\right)^\beta m(y-y',\I_i). 
  \end{split}
\end{equation}
Therefore:
 \begin{equation}
\Vert \mathbf{W}  (\hat{y}-\hat{y}') \Vert_1 = \sum_{\I \in \T}  \left(\frac{\vert \I_i \vert}{n_\X}\right)^\beta \vert m(y-y',\I) \vert \;\; \hfill \square
\end{equation}
\end{proof}
Note that the proposed metric is equivalent to the $l_1$ distance between vectors of higher-dimensionality than the original dimension of the vectors. 
However, by weighting the coefficients with $\mybf{W}$, the effective dimension of the new vectors is typically smaller than the original dimensionality, as the weights rapidly decrease to zero based on the folder size and the choice of $\beta$. 
For positive values of $\beta$, the entries corresponding to the large folders dominate $\hat{y}$, while entries corresponding to small folders tend to zero.
This trend is reversed for negative values of $\beta$, with elements corresponding to small folders dominating $\hat{y}$ while large folders are suppressed.
In both cases, a threshold can be applied to $\hat{y}$ or $\hat{\mybf{Z}}$ so as to discard entries with low absolute values. 
Thus, the transforms project the data onto a low-dimensional space of either coarse or fine structures.
Also, note that interpreting the metric as the $l_1$ distance in the averaging transform space enables us to apply approximate nearest-neighbor search algorithms suitable for the $l_1$ distance~\cite{Arya:1998,Yi2000}.
This allows to analyze larger datasets via a sparse affinity matrix.

Defining the metric in the transform space enables us to easily generalize the metric to a joint-tree metric defined for a joint pair of trees $\{\T_\X,\T_\Y\}$ (Sec.~\ref{subsec:joint}), to incorporate several trees over the features $\{\T_\X\}^{n_\T}$ in a multi-tree metric via the multi-tree transform (Sec.~\ref{sec:multi_metric}), and to seamlessly introduce weights on the transform coefficients by setting the elements of $\mybf{W}$ (Sec.~\ref{subsec:weights}). Python code implementing our approach is available at~\cite{ourcode}.

\vspace{-0.2cm}
\subsection{Joint-tree Metric}
\label{subsec:joint}
The tree-based transforms and metrics can be generalized to analyzing rank-n tensor datasets.
We briefly present the joint-tree metric to demonstrate that the proposed transforms are not limited to just 2D matrices, but rather can be extended to processing and organizing tensor datasets. An example of such an application was presented in~\cite{Mishne2015b}.

In~\cite{Mishne2015b} we proposed a 2D metric given a pair of partition trees in the setting of organizing a rank-3 tensor. 
We reformulate this metric in the transform space by generalizing the tree-based metric to a joint-tree metric using the coefficients of the joint-tree transform.
Given a partition tree $\mathcal{T}_\X$ on the features and a partition tree $\mathcal{T}_\Y$ on the observations, the distance between two matrices $\mybf{Z}_1$ and $\mybf{Z}_2$ is defined as
\begin{equation}
\label{eq:2demd}
d_{\mathcal{T}_\X,\mathcal{T}_\Y}(\mybf{Z}_1 , \mybf{Z}_2) = \sum_{\substack{\mathcal{I} \in \mathcal{T}_\X \\ \mathcal{J} \in \mathcal{T}_\Y}} \vert m(\mybf{Z}_1-\mybf{Z}_2, \mathcal{I} \times \mathcal{J})\vert \frac{\vert \I \vert^{\beta_\X} \vert \mathcal{J} \vert^{\beta_\Y} } {n_\X^{\beta_\X}  {n_\Y}^{\beta_\Y}}.
\end{equation}
The value $m(\mybf{Z}, \mathcal{I} \times \mathcal{J})$ is the mean value of a matrix $\mybf{Z}$ on the joint folder $\mathcal{I} \times \mathcal{J} = \{ (x,y)\; \vert\; x\in \mathcal{I}, y\in \mathcal{J} \}$:
\begin{equation}
m(\mybf{Z}, \mathcal{I} \times \mathcal{J}) = \frac{1}{\vert \mathcal{I} \vert \vert \mathcal{J} \vert } \sum_{x \in \mathcal{I}, y\in \mathcal{J}} \mybf{Z}[x,y].
\end{equation}
Theorem~\ref{thm:emd} can be generalized to a 2D transform applied to 2D matrices.
\begin{corollary}
\cite[Corollary 4.2]{Mishne2015b}
 The joint-tree metric~(\ref{eq:2demd}) between two matrices given a partition tree $\mathcal{T}_\X$ on the features and a partition tree $\mathcal{T}_\Y$ on the observations is equivalent to the $l_1$ distance between the weighted 2D multiscale transform of the two matrices:
\begin{equation}
d_{\mathcal{T}_\X,\mathcal{T}_\Y}(\mybf{Z}_1 , \mybf{Z}_2) = \Vert \mybf{W}_\X\mybf{M}_\X (\mybf{Z}_1 - \mybf{Z}_2)\mybf{M}_\Y^T \mybf{W}_\Y\Vert_1.
\end{equation}
\end{corollary}

\subsection{Multi-tree Metric}
\label{sec:multi_metric}
The definition of the metric in the transform domain enables a simple extension to a metric derived from a multi-tree composition. 
Given a set of multiple trees $\{\T_t\}_{t=1}^{n_\T}$ defined on the features $\X$ as in Sec.~\ref{sec:multitree}, we define a multi-tree metric using the multi-tree averaging tree transform as:
\begin{equation}
\label{eq:multimetric}
\begin{split}
d_{\{\T\}}(y,y') = \Vert \Wall \Mall (y-y') \Vert_1,
\end{split}
\end{equation}
where $\Wall$ is a diagonal matrix whose elements are $\left(\frac{\vert \I_i \vert}{n_\X}\right)^\beta$ for all $\I \in \T$ and for all trees in $\{\T_t\}_{t=1}^{n_\T}$.
This metric is equivalent to averaging the single tree metrics: 
\begin{equation}
\begin{split}
d_{\{\T\}}(y,y') & = \Vert \Wall \Mall (y-y') \Vert_1 \\
& = \frac{1}{n_\T} \sum_\T \Vert \mathbf{W}_\T \mathbf{M}_\T (y-y') \Vert_1\\
&= \frac{1}{n_\T} \sum_\T d_\T(y,y'). 
\end{split}
\end{equation}
Note that in contrast to the joint-tree metric, which incorporates a pair of trees over both the features \emph{and} the observations, here we are referring to a set of trees defined only for the features, or only for the observations.

A question that arises is how to construct multiple trees?
For matrix denoising in a bi-organization setting, Ankenman~\cite{Ankenman2014} applies a spin-cycling procedure: constructing many trees by randomly varying the parameters in the partition tree construction algorithm.
Multiple trees can also be obtained by initializing the bi-organization with different metric choices for $d_\X^{(0)}(x,x')$ (step~\ref{step:aff} in Alg.~\ref{alg:biorg}), e.g., Euclidean, correlation, etc.
Another option, which we demonstrate experimentally on real data in Sec.~\ref{sec:results}, arises when we have multiple data sets of observations with the same set of features, or multiple data sets with the same observations but different features as in multi-modal settings. 
In such cases, we construct a partition tree for each dataset separately and then combine them using the multi-tree metric. 

\subsection{Local Refinement}
We propose a new approach to constructing multiple trees, leveraging the partition of the data during the bi-organization procedure.
This approach is based on a local refinement of the partition trees, which results in a smoother organization of the data.
The bi-organization method is effective when correlations exist among both observations and features, by revealing a hierarchical organization that is meaningful for all the data together. 
Yet, since the bi-organization approach is global and takes all observations and all features into account, it needs to achieve the best organization \emph{on average}.
However, the correlations between features may differ among \emph{sub}-populations in the data, i.e. the correlations between features depend on the set of observations taken into account (and vice-versa).

For example, consider a dataset of documents where the observations $\Y$ are documents belonging to different categories, the features $X$ are words and $\mybf{Z}(x,y)$ indicates whether a document $y$ contained a word $x$.
Grouping the words into disjoint folders forces a single partition of the vocabulary that disregards words that belong to more than one conceptual group of words. 
These connections could be revealed by taking into account the context, i.e. the subject of the documents.
By diving the documents into a few contextual clusters, and calculating a local tree on the words $\T_\X$ for each such cluster, the words are grouped together conceptually according to the local category.
The word ``field'' for example will be joined with different neighbors, depending on whether the analysis is applied to documents belonging to ``agriculture'', ``mathematics'' or ``sports''.

\begin{algorithm}[t]
\caption{Bi-organization Algorithm \cite[Sec. 5.3]{Ankenman2014}}
\label{alg:biorg}
\algsetup{indent=1.5em}
\begin{algorithmic}[1]
\INIT
\INPUT Dataset $\mathbf{Z}$ of features $\X$ and observations $\Y$
\STATE Starting with features $\X$
\STATE \label{step:aff} \hspace{0.5cm} Calculate initial metric $d_\X^{(0)}(x,x')$ 
\STATE \label{step:tree} \hspace{0.5cm} Calculate initial flexible tree  $\mathcal{T}_\X^{(0)}$. 
\ANALYSIS
\INPUT Flexible tree on features $\mathcal{T}_\X^{(0)}$, weight function on tree folders $\mathbf{W}[i,i]=\omega(\I_i)$ 
\FOR{$n \geq 1$} \label{step:beginfor}
\STATE \label{step:dist} Given tree $\mathcal{T}_\X^{(n)}$, calculate tree metric between observations $d_{\mathcal{T}_\X}^{(n)}(y,y')=\Vert \mathbf{W_{\X}M_{\X}}(y-y')\Vert_1$
\STATE \label{step:tree2} Calculate flexible tree on the observations $\mathcal{T}_\Y^{(n)}$. 
\STATE \label{step:iter} Repeat steps \ref{step:dist}-\ref{step:tree2} for the features $\X$ given $\mathcal{T}_\Y^{(n)}$  and obtain $\mathcal{T}_\X^{(n+1)}$.
\ENDFOR \label{step:endfor}
\end{algorithmic}
\end{algorithm}

\begin{algorithm}[t]
\caption{Bi-organization local refinement}
\label{alg:refine}
\algsetup{indent=1.5em}
\begin{algorithmic}[1]
\INPUT Dataset $\mathbf{Z}$, observation tree $\T_\Y$
\STATE Choose level $l$ in tree $\T_\Y$
\FOR{$j \in\{1,...,n(l)\} $}
\STATE Set $\omega(\J_i) = 1\;\; \forall \J_i \subseteq \mathcal{J}_{l,j}$, otherwise $\omega(\J_i)=0$.
\STATE  Calculate initial affinity on features for subset of observations as weighted tree-metric $d^{(0)}(x,x') = d_{\mathcal{T}_\Y}(x,x';\omega(\J_j))$
\STATE Calculate initial flexible tree on features $\mathcal{T}_\X^{(0)}$
\STATE Perform iterative analysis (steps \ref{step:beginfor}-\ref{step:endfor} in Alg.~\ref{alg:biorg}) for $\mathbf{Z}$ on $\X$ and $\widetilde{\Y}= \{y \;\vert\; y \in \mathcal{J} \in \T_\Y\}$.
\ENDFOR
\STATE Merge observation trees $\{\T_{\widetilde{\Y}_j}\}^{n(l)}$ back into global tree $\T_\Y$
\OUTPUT Refined observation tree $\T_\Y$, Set of feature trees $\{\T_\X\}_{i=1}^{n(l)}$
\end{algorithmic}
\end{algorithm}

Therefore, we propose to take advantage of the unsupervised clustering obtained by the partition tree on the observations $\T_\Y$, and apply a localized bi-organization to folders of observations.
Formally, we apply the bi-organization algorithm to a subset of $\mybf{Z}$ containing all features $\X$ and a subset of observations belonging to the same folder $\widetilde{\Y} = \{y \;\vert\; y \in \mathcal{J} \in \T_\Y\}$.
This local bi-organization results in a pair of trees: a local tree $\T_{\widetilde{\Y}}$ organizing the subset of observations $\widetilde{\Y}$, and a feature tree $\T_\X$ that organizes all the features $\X$ based on this subset of  observations that share the same local structure, rather than the global structure of the data.
This reveals the correlations between features for this sub-population of the data, and provides a \emph{localized} visualization and exploratory analysis for subsets of the data discovered in an unsupervised manner.
This is meaningful when the data is unbalanced and a subset of the data differs drastically from the rest of the data, e.g., due to anomalies.

We propose a local refinement of the bi-organization as follows. We select a single layer $l$ of the observations tree $\T_\Y$, and perform a separate localized organization for each folder $\mathcal{J}_{l,j} \in \mathcal{P}_l, \;\; j\in\{1,...,n(l)\}$.
We thus obtain $n(l)$ local observation trees $\{\T_{\widetilde{\Y}_j}\}_{j=1}^{n(l)}$, which we then merge back into one global tree, with refined partitioning.
Merging is performed by replacing the branch in $\T_\Y$ whose root is $\mathcal{J}_{l,j}$, i.e. $\{\mathcal{J} \in \T_\Y \vert \mathcal{J} \subseteq \mathcal{J}_{l,j} \}$, with the local observation tree $\T_{\widetilde{\Y}_j}$. 
In addition, we obtain a set of several corresponding trees on the full set of features $\{\T_\X\}^{n(l)}$, which we can use to calculate a multi-tree metric~(\ref{eq:multimetric}).
Our local refinement algorithm is presented in Alg.~\ref{alg:refine}.
Applying this algorithm to refine the global structures of both $\T_\Y$ and $\T_\X$ results in a smoother bi-organization of the data.

We typically apply the refinement to a high level of the tree since at these levels large clusters of distinct sub-populations are grouped together, and their separate analysis will reveal their local organization.  
The level can be chosen by applying the difference transform and selecting a level at which the folders grouped together are heterogeneous, i.e. their mean significantly differs from the mean of their parent folder.  

Note that this approach is unsupervised and relies on the data-driven organization of the data.
However, this approach can also be used in a supervised setting, when there are labels on the observations.
Then we calculate a different partition tree on the features for each separate label (or sets of labels) of the observations, revealing the hierarchical structure of the features for each label. This will be explored in future work.

\subsection{Weight Selection}
\label{subsec:weights}
The calculation of the metric depends on the weight attached to each folder.
We generalize the metric such that the weight is $\mathbf{W}[i,i]=\omega(\I_i)$, where $\omega(\I_i)>0$ is a weight function associated with folder $\I_i$.
The weights can incorporate prior smoothness assumptions on the data, and also enable to enhance either coarse or fine structures in the similarity between samples. 

The choice $\omega(\I_i)=\left(\frac{\vert \I_i \vert}{n_\X}\right)^\beta$ in~\cite{Leeb2013} makes the tree-based metric~(\ref{eq:emd}) equivalent to EMD, i.e., the ratio of EMD to the tree-based metric is always between two constants.
The parameter $\beta$ weights the folder by its relative size in the tree, where $\beta>0$ emphasizes coarser scales of the data, while $\beta <0$ emphasizes differences in fine structures.

Ankenman~\cite{Ankenman2014} proposed a slight variation to the weight also encompassing the tree structure:
 \begin{equation}
 \label{eq:weight2}
 \omega(\I_i) = 2^{-\alpha l(\I_i)} \left(\frac{\vert \I_i \vert}{n_\X}\right)^\beta,
 \end{equation}
 where $\alpha$ is a constant and $l(\I_i)$ is the level at which the folder $\I_i$ is found in $\mathcal{T}$.
 The constant $\alpha$ weights all folders in a given level equally. 
 Choosing $\alpha=0$ resorts to the original weight.
The structure of the trees can be seen as an analogue to a frequency decomposition in signal processing, where the support of a folder is analogous to a certain frequency. Moreover, since high levels of the tree typically contain large folders, they correspond to low-pass filters. Conversely, lower levels of the tree correspond to high-pass filters as they contain many small folders.
Thus setting $\alpha>0$ corresponds to emphasizing low frequencies whereas $\alpha <0$ corresponds to enhancing high frequencies.
 In an unbalanced tree, where a small folder of features remains separate for \emph{all} levels of the tree (an anomalous cluster of features), $\alpha$ can be used to enhance the importance of this folder, as opposed to $\beta$, which would decrease its importance based on its size. 

We propose a different approach. 
Instead of weighting the folders based on the structure of the tree, which requires a-priori assumptions on the optimal scale of the features or the observations, we set the folders weights based on their content.
By applying the difference transform to the data, we obtain a measure for each folder defining how homogeneous it is.
This reduces the number of parameters in the algorithm, which is advantageous in the unsupervised problem of bi-organization.
We calculate for each folder, the norm of its difference on the dataset $\mybf{Z}$:
\begin{equation}
\label{eq:weight_mar}
\begin{split}
\omega(\I_i) & = \left(\sum_y \left((\mybf{\Delta}_\X\mybf{Z})[i,y] \right)^2 \right)^{1/2}\\
& = \left( \sum_y \left(\sum_x (m(y(x), \I_{l,i}) - m(y(x),\I_{l+1,j}) \right)^2 \right)^{1/2},
\end{split}
\end{equation}
where $\I_{l,i}\subset\I_{l+1,j}$.
This weight is high when $\I_{l,i}\nsim\I_{l+1,j}$.
This means that the parent folder joining $\I_{l,i}$ with other folders contains non-homogeneous ``populations''. 
Therefore, assigning a high weight to $\I_{l,i}$ places importance on differentiating these different populations.

The localized refinement procedure in Alg.~\ref{alg:refine} can also be formalized as assigning weights $\omega(\I)$ in the tree metric.
We set all weights containing a branch of the tree (a folder and all its sub-folders) to 1 and set all other weights to zero:
\begin{equation}
\omega(\I_i) = \left\{
  \begin{array}{lr}
    1, & \I_i \subseteq \I_j\\
    0, & \textrm{otherwise},
  \end{array}
\right.
\end{equation}
where $\I_j$ is the root folder of the branch.
Thus, using these weights, the metric is calculated based only on a subset of the observations $\tilde{\Y}$. 
This metric can initialize a bi-organization procedure of a subset of $\mybf{Z}$ containing $\X$ and $\tilde{\Y}$.

\subsection{Coherence}
To assess the smoothness of the bi-organization stemming from the constructed partition trees, a coherency criterion was proposed in~\cite{Gavish2012}.
The coherency criterion is given by
\begin{equation}
\label{eq:coherency}
C(\mybf{Z};\T_\X,\T_\Y) = \frac{1}{n_\X n_\Y}\Vert\mybf{\Psi}_\X \mybf{Z}\mybf{\Psi}_\Y^T\Vert_1,
\end{equation}
where $\mybf{\Psi}$ is a Haar-like orthonormal basis proposed by Gavish, Nadler and Coifman~\cite{Gavish2010} in the settings of partition trees, and it depends on the structure of a given tree.
This criterion measures the decomposition of the data in a bi-Haar-like basis induced by two partition trees $\T_\X$ and $\T_\Y$:  $\mybf{\Psi}_\X \mybf{Z}\mybf{\Psi}_\Y^T$.
The lower the value of $C(\mybf{Z};\T_\X,\T_\Y)$, the smoother the organization is in terms of satisfying the mixed H\"{o}lder condition~(\ref{eq:mixedH}).

Minimizing the coherence can be used as a stopping condition for the bi-organization algorithm presented in Alg.~\ref{alg:biorg}.
The bi-organization continues as long as $C(\mybf{Z};\T_\X^{(n)},\T_\Y^{(n)}) < C(\mybf{Z};\T_\X^{(n-1)},\T_\Y^{(n-1)})$~\cite{Gavish2012}. 
However, we have empirically found that the iterative process typically converges within only few iterations. 
Therefore, in our experimental results we perform $n=2$ iterations.

\section{Experimental Results}
\label{sec:results}
Analysis of cancer gene expression data is of critical importance in jointly identifying subtypes of cancerous tumors and genes that can distinguish the subtypes or indicate a patient's long-term survival.
Identifying a patient's tumor subtype can determine the course of treatment, such as recommendation of hormone therapy in some subtypes of breast cancer, and is a an important step toward the goal of personalized medicine. 
Biclustering of breast cancer data has identified sets of genes whose expression levels categorize tumors into five subtypes with distinct survival outcomes~\cite{Sorlie2001}: Luminal A, Luminal B, Triple negative/basal-like, HER2 type and ``Normal-like''.
Related work has aimed to classify samples into each of these subtypes or identify other types of significant clusters based on gene expression, clinical features and DNA copy number analysis~\cite{Parker2009,Curtis2012,Cancer2012,Milioli2016}.
The clustered dendrogram obtained by agglomerative hierarchical clustering of the genes and the subjects is widely used in the analysis of gene expression data.
However, in contrast to our approach, hierarchical clustering is usually applied with a metric, such as correlation, that is global and linear, and does not take into account the structure revealed by the multiscale tree structure of the other dimension. 
Conversely, our approach enables us to iteratively update both the tree and metric of the subjects based on the metric for the genes, and update the tree and metric of the genes based on the metric for the subjects.

\begin{figure*}[t]
\centering{\includegraphics[width=0.78\linewidth]{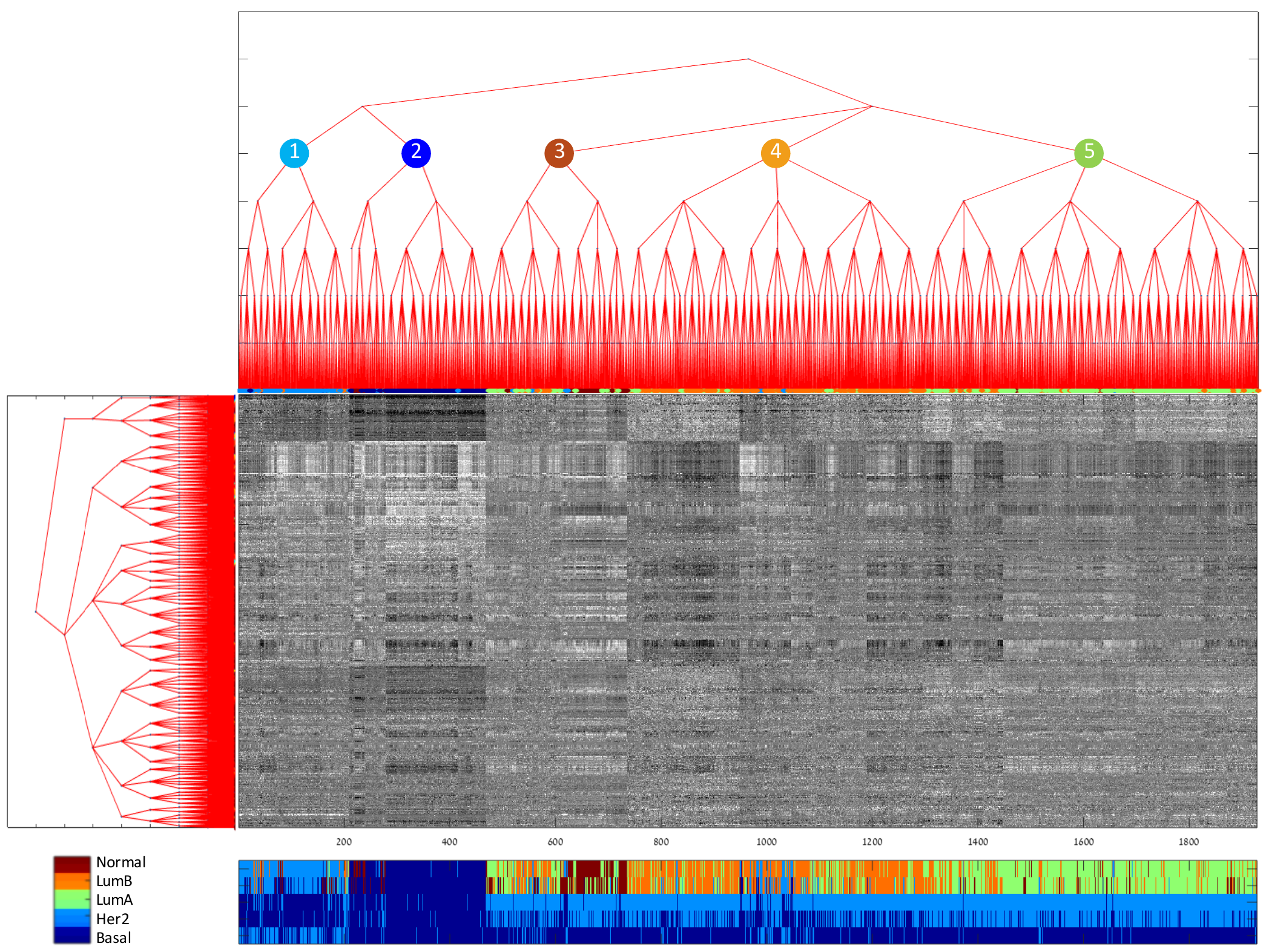}}
\caption{Global bi-organization of the METABRIC dataset. The samples (columns) and genes (rows) have been reordered so they correspond to the leaves of the two partition trees.
Below the organized data are clinical details for each of the samples: two types of breast cancer subtype labels (refined~\cite{Milioli2016} and PAM50~\cite{Parker2009}) hormone receptor status (ER, PR) and HER2 status.
}
\label{fig:biorg_init}
\end{figure*}

\begin{figure}[th]
\centering{\includegraphics[width=0.99\linewidth]{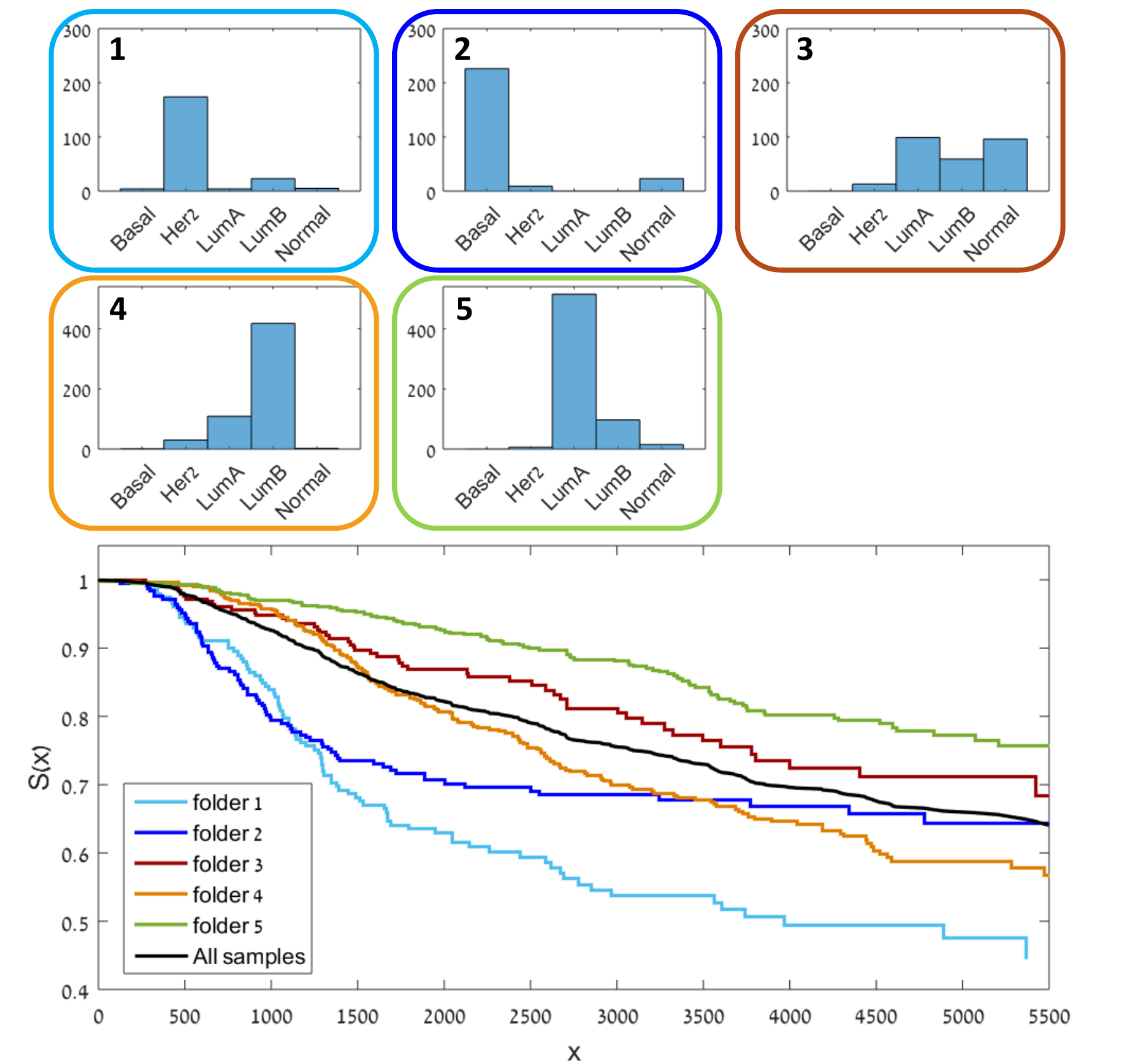}}
\caption{(top) Histograms of folders in sample tree of METABRIC. The color of the border corresponds to the circles in the tree. (bottom) Survival curves for each folder.}
\label{fig:surv_local}
\end{figure}

We analyze three breast cancer gene expression datasets, where the features are the genes and the observations are the tumor samples.
The first dataset is the METABRIC dataset, containing gene expression data for 1981 breast tumors~\cite{Curtis2012} collected with a gene expression microarray.
We denote this dataset $Z_{\textrm{M}}$, and its set of samples $\Y_{\textrm{M}}$.
The second dataset, $Z_{\textrm{T}}$, is taken from The Cancer Genome Atlas (TCGA) Breast Cancer cohort~\cite{TCGAdata} and consists of 1218 samples, $\Y_{\textrm{T}}$.
This dataset was profiled using RNA sequencing, which is a newer and more advanced gene expression technology. 
The third dataset $Z_{\textrm{B}}$ (BRCA-547)~\cite{Cancer2012}, comprising of 547 samples $\Y_{\textrm{B}}$, was acquired with microarray technology. 
These 547 samples are also included in the TCGA cohort, but the gene expression was profiled using a different technology.

We selected $\X$ to be the 2000 genes with the largest variance in METABRIC from the original collection of~$\sim40000$ gene probes.
In related work, the analyzed genes were selected in a supervised manner based on prior knowledge or statistical significance in relation to patient survival time~\cite{Perou2000,Sorlie2001,Parker2009,Curtis2012,Milioli2016}.
Here we present results of a purely unsupervised approach aimed at exploratory analysis of high-dimensional data, and we do not use the survival information or subtypes labels in either applying our analysis or for gene selection, but only in evaluating the results.
In the remainder of this section we present three approaches in which the tree transforms and metrics are applied for the purpose of unsupervised organization of gene expression data.

Regarding implementation, in this application we use flexible trees~\cite{Ankenman2014} to construct the partition trees in the bi-organization.
We initialize the bi-organization with a correlation affinity on the genes ($d_\X^{(0)}(x,x')$ in Alg.~\ref{alg:biorg}, Step~\ref{step:aff}), which is commonly used in gene expression analysis.

\subsection{Subject Clustering}
\label{sec:res_self}
We begin with a global analysis of all samples of the METABRIC data using the bi-organization algorithm presented in Alg.~\ref{alg:biorg}.
We perform two iterations of the bi-organization using the tree-based metric with the data-driven weights defined in~(\ref{eq:weight_mar}).
The organized data and corresponding trees on the samples and on the genes are shown in Fig.~\ref{fig:biorg_init}.
The samples and genes have been reordered such that they correspond to the leaves of the two partition trees.
Below the organized data we provide clinical details for each of the samples: two types of breast cancer subtype labels, the refined labels introduced in~\cite{Milioli2016} and the standard PAM50 subtypes~\cite{Parker2009}, hormone receptor status (ER, PR) and HER2 status.
We analyze the folders of level $l=5$ on the samples tree, which divides the samples into five clusters (the folders are marked with numbered colored circles).

In Fig.~\ref{fig:surv_local} we present histograms of the refined subtype labels for each of the numbered folders in the samples tree, and plot the disease-specific survival curve of each folder in the bottom graph. 
The histograms of each folder is surrounded by a colored border corresponding to the colored circle indicating the relevant folder in the tree in Fig.~\ref{fig:biorg_init}.
Note that the folders do not just separate data according to subtype as in the dark blue and light blue folders (Basal and Her2 respectively), but also separate data according to the survival rates.
If we compare the orange and green folders that are grouped in the same parent folder, both contain a mixture of Luminal A and Luminal B, yet they have distinctive survival curves.
The p-value of this separation using the log-rank test~\cite{Peto1972} was $4.35 \times 10^{-21}$.

\begin{table}[t]
\centering
\caption{METABRIC Self organization}
\label{tab:internal}
\scalebox{0.95}{
\begin{tabular}{l|l|l|l|l|}
\cline{2-5}
                                 & RI   & ARI  & VI   & p-value              \\ \hline 
\multicolumn{1}{|l||}{weighted}   & \bf{0.79} & \bf{0.45} & \bf{1.48} & ${4.35\times10^{-21}}$ \\ \hline
\multicolumn{1}{|l||}{$(\alpha,\beta)=(0,0)$}  & 0.72 & 0.30 & 1.77 & $1.11\times10^{-17}$ \\ \hline
\multicolumn{1}{|l||}{$(\alpha,\beta)=(0,-1)$} & 0.72 & 0.23 & 1.98 & $8.48\times10^{-10}$ \\ \hline
\multicolumn{1}{|l||}{$(\alpha,\beta)=(0,1)$}  & 0.69 & 0.20 & 1.94 & $1.46\times10^{-12}$  \\ \hline
\multicolumn{1}{|l||}{$(\alpha,\beta)=(-1,0)$} & 0.74 & 0.30 & 1.84 & $1.11\times10^{-16}$ \\ \hline
\multicolumn{1}{|l||}{$(\alpha,\beta)=(0.5,0)$}  & 0.72 & 0.26 & 1.90& $5.23\times10^{-11}$  \\ \hline \hline
\multicolumn{1}{|l||}{DTC~\cite{Langfelder2008}}  & 0.74 & 0.19 & 2.45 & $\bf{5.54\times10^{-22}}$  \\ \hline
\multicolumn{1}{|l||}{sparseBC~\cite{Tan2014}}  & 0.76 & 0.33 & 1.74 & $2.6\times10^{-19}$  \\ \hline
\end{tabular}}
\end{table}

\begin{table}[t]
\centering
\caption{BRCA-547 Self organization}
\label{tab:brca_internal}
\scalebox{0.95}{
\begin{tabular}{l|l|l|l|l|}
\cline{2-5}
                                 & RI   & ARI  & VI   & p-value              \\ \hline
\multicolumn{1}{|l||}{weighted}   & {0.75} & \bf{0.38} & 1.38 & \bf{0.0004} \\ \hline
\multicolumn{1}{|l||}{$(\alpha,\beta)=(0,0)$}  & {0.75} & 0.37 & 1.39 & $0.0073$ \\ \hline
\multicolumn{1}{|l||}{$(\alpha,\beta)=(0,-1)$} & 0.74 & 0.36 & 1.37 & $0.0028$ \\ \hline
\multicolumn{1}{|l||}{$(\alpha,\beta)=(0,1)$}  & 0.72 & 0.35 & \bf{1.33} & $0.0773$  \\ \hline
\multicolumn{1}{|l||}{$(\alpha,\beta)=(-1,0)$} & 0.74 & 0.34 & 1.56 & $0.0010$ \\ \hline
\multicolumn{1}{|l||}{$(\alpha,\beta)=(0.5,0)$}  & 0.74 & 0.35 & 1.45& $0.0130$  \\ \hline  \hline
\multicolumn{1}{|l||}{DTC~\cite{Langfelder2008}}  & 0.75 & 0.35 & 1.63& $0.0853$  \\ \hline
\multicolumn{1}{|l||}{sparseBC~\cite{Tan2014}}  & \bf{0.76} & \bf{0.38} & 1.49 & $0.0269$  \\ \hline
\end{tabular}}
\end{table}

We next compare our weighted metric~(\ref{eq:weight_mar}) to the original EMD-like metric~(\ref{eq:emd}), using different values of $\beta$ and $\alpha$ in (\ref{eq:weight2}). These values were chosen in order to place different emphasis of the transform coefficients depending on the support of the corresponding folders or the level of the tree.
The values of $\beta$ enable to emphasize large folders ($\beta=1$), small folders ($\beta=-1$) and weighting all folders equally ($\beta=0$).
The values of $\alpha$ either emphasize high levels of the tree ($\alpha=0.5$), low levels of the tree ($\alpha=-1$) or weighting all levels equally ($\alpha=0$). 

We also compare to two other biclustering methods.
The first is the dynamic tree cutting (DTC)~\cite{Langfelder2008} applied to a hierarchical clustering dendrogram obtained using mean linkage and correlation distance (a popular choice in gene expression analysis).
The second is the sparse biclustering method~\cite{Tan2014}, where the authors impose a sparse regularization on the mean values of the estimated biclsuters (assuming the mean of the dataset is zero).
Both algorithms are implemented in R: package \texttt{dynamicTreeCut} and package \texttt{sparseBC}, respectively. 

We evaluate our approach by both measuring how well the obtained clusters represent the cancer subtypes, and estimating the statistical significance of the survival curves of the clusters.
We compare the clustering of the samples relative to the refined subtype labels~\cite{Milioli2016} using three measures: the Rand index (RI)~\cite{Rand1971}, the adjusted Rand index (ARI)~\cite{Hubert1985}, and the variation of information (VI)~\cite{Meila2007}. 
The RI and ARI measure the similarity between two clusterings (or partitions) of the data.
Both measures indicate no agreement between the partitions by 0 and perfect agreement by 1, however ARI can return negative values for certain pairs of clusterings.
The third measure is an information theoretic criterion, where 0 indicates perfect agreement between two partitions. 
Finally, we perform survival analysis using Kaplan-Meier estimate~\cite{Kaplan1958} of disease-specific survival rates of the samples, reporting the p-value of the log-rank test~\cite{Peto1972}. A brief description of these statistics is provided in Appendix~\ref{app:KM}.

We select clusters by partitioning the samples into the folders $\mathcal{J}$ of the samples tree $\T_\X$, at a single level $l$ of the tree which divides the data into 4-6 clusters (typically level $L-2$ in our experiments). This follows the property of flexible trees that the level at which folders are joined is meaningful across the entire dataset, as for each level the distances between joined folders are similar.
For other types of tree construction algorithms, alternative methods can be used to select clusters in the tree, such as SigClust used in~\cite{Cancer2012}.

Results are presented in Table~\ref{tab:internal} for the METABRIC dataset and in Table~\ref{tab:brca_internal} for the BRCA-547 dataset.
For the METABRIC dataset, using the weighted metric achieves the best results compared to the other weight selections, in terms of both clustering relative to the ground-truth labels and the survival curves of the different clusters (note these two criteria do not always coincide).
While DTC achieves the lowest p-value overall, it has very poor clustering results compared to the ground-truth labels (lowest ARI and highest VI). 
The weighted metric out-performed the sparseBC method, which has second-best performance for the clustering measures, and third-lowest p-value.
For the BRCA-547 dataset, the weighted metric achieves the best clustering in terms of the ARI measure and has the lowest p-value.
For the VI measure, the clustering by the weighted metric was slightly larger but comparable to that of the lowest score.
On this dataset, DTC performed poorly with highest VI and p-value. 
The sparseBC method achieved good clustering with highest RI and ARI measures, but had a high p-value and VI compared to the performance of our bi-organization method.

The results indicate that the data-driven weighting achieves comparable if not better performance, than both using the tree-dependent weights and competing biclustering methods. 
Thus, the data-driven weighting provides an automatic method to set appropriate weights on the transform coefficients in the metric.
Our method is completely data-driven, as opposed to the sparseBC method which requires as input the number of features and observations to decompose the data into. (We used the provided computationally expensive cross-validation procedure to select the best number of clusters in each dimension).
In addition, our approach provides a multiscale organization, whereas sparseBC yields a single-scale decomposition of the data.
The DTC is a multiscale approach, however as it relies on hierarchical clustering it does not take into account the dendrogram in the other dimension. The performance may be improved by using dendrograms in our iterative approach, instead of the flexible trees (this is further discussed below).

\subsection{Local refinement}
In Table~\ref{tab:coherency} we demonstrate the improvement gained in the organization by applying the local refinement to the partition trees, where we measure the smoothness of the organized data using the coherency criterion~(\ref{eq:coherency}).
We perform bi-organization for different values of $\beta$ and $\alpha$ as well as the weighted metric, and compare 4 organizations: 1) Global organization; 2) Refined organization of only the genes tree $\T_\X$; 3) Refined organization of only the samples tree $\T_\Y$; and 4) Refined organization of both the features and the samples (refined $\T_\X$ and $\T_\Y$).
Applying the refined local organization to both the genes and the samples, yields the best result with regard to the smoothness of the bi-organization.
We also examined the the effect of the level of the tree on which the refinement is performed for $l\in\{5,6,7\}$ for both trees, and the improvement gained by refinement was of the same order for all combinations.
The results demonstrate that regardless of the weighting (data-driven or folder dependent), the refinement procedure improves the coherency of the organization.

\begin{table}[t]
\centering
\caption{Coherency of refined bi-organization}
\label{tab:coherency}
\scalebox{0.9}{
\begin{tabular}{l|l|l|l|l|}
\cline{2-5}
                                 & \begin{tabular}[c]{@{}l@{}}Global $\T_\X$  \\ and $\T_\Y$\end{tabular} & \begin{tabular}[c]{@{}l@{}}Refined\\  $\T_\X$\end{tabular} & \begin{tabular}[c]{@{}l@{}}Refined \\ $\T_\Y$\end{tabular} & \begin{tabular}[c]{@{}l@{}}Refined \\ $\T_\X$, $\T_\Y$\end{tabular} \\ \hline
\multicolumn{1}{|l||}{weighted}   & 0.7039                                                 & 0.6103                                                & 0.5908                                                 & \bf{0.5463}                                                     \\ \hline
\multicolumn{1}{|l||}{$(\alpha,\beta)=(0,0)$}  & 0.7066                                                  & 0.6107                                                 & 0.5928                                                & \bf{0.5480}                                                     \\ \hline
\multicolumn{1}{|l||}{$(\alpha,\beta)=(0,-1)$} & 0.7051                                                  & 0.6118                                                 & 0.5921                                                 & \bf{0.5472}                                                     \\ \hline
\multicolumn{1}{|l||}{$(\alpha,\beta)=(0,1)$}  & 0.7028                                                  & 0.6130                                                 & 0.5972                                                 & \bf{0.5668}                                                     \\ \hline
\multicolumn{1}{|l||}{$(\alpha,\beta)=(-1,0)$}  & 0.7051                                                  & 0.6119                                                 & 0.5927                                                & \bf{0.5487}                                                     \\ \hline
\multicolumn{1}{|l||}{$(\alpha,\beta)=(0.5,0)$}  & 0.7075                                                  & 0.6141                                                & 0.5934                                                 & \bf{0.5497}                                                     \\ \hline
\end{tabular}}
\end{table}

\subsection{Bi-organization with multiple datasets}
Following the introduction of gene expression profiling by RNA sequencing, an interesting scenario is that of two datasets profiled using different technologies, one using microarray and the other RNA sequencing.
Consider, for example, the METABRIC dataset $Z_{\textrm{M}}$ and the TCGA dataset $Z_{\textrm{T}}$, which share the same features $\X$ (in this case genes), but collected for two different sample sets, $\Y_{\textrm{M}}$ and $\Y_{\textrm{T}}$ respectively.
In this case, the gene expression profiles have different dynamic range and are normalized differently, and the samples cannot be analyzed together simply by concatenating the datasets.
However, the hierarchical structure we learn on the genes, which defines a multiscale clustering of the genes, is informative regardless of the technique used to acquire the expression data.

Thus, the gene metric learned from one dataset can be applied seamlessly to another dataset and used to organize its samples due to the coupling between the genes and the samples.
We term this ``external-organization'', and demonstrate how it organizes the METABRIC dataset $Z_{\textrm{M}}$ using the TCGA dataset $Z_{\textrm{T}}$. 
We first apply the bi-organization algorithm to organize $Z_{\textrm{T}}$, and then we derive the gene tree-based metric $d_{\T_\X}$ from the constructed tree on the genes $\T_\X$.
This metric is then used to a construct a new tree $\T_{\Y}$ on the samples set $\Y_{\textrm{M}}$ of $Z_{\textrm{M}}$.

\begin{table}[t]
\centering
\caption{METABRIC External organization}
\label{tab:external}
\scalebox{0.9}{
\begin{tabular}{l|l|l|l|l|}
\cline{2-5}
                                 & RI   & ARI  & VI   & p-value              \\ \hline
\multicolumn{1}{|l||}{weighted}   & \bf{0.74} & \bf{0.30} & \bf{1.77} & $\bf{3.71\times10^{-19}}$ \\ \hline
\multicolumn{1}{|l||}{$(\alpha,\beta)=(0,0)$}  & 0.73 & 0.29 & 1.87 & $7.78\times10^{-16}$ \\ \hline
\multicolumn{1}{|l||}{$(\alpha,\beta)=(0,-1)$} & 0.72 & 0.26 & 1.87 & $1.77\times10^{-16}$ \\ \hline
\multicolumn{1}{|l||}{$(\alpha,\beta)=(0,1)$}  & 0.73 & 0.28 & 1.83 & $4.25\times10^{-14}$  \\ \hline
\multicolumn{1}{|l||}{$(\alpha,\beta)=(-1,0)$} & 0.72 & 0.27 & 1.89 & $7.02\times10^{-6}$ \\ \hline
\multicolumn{1}{|l||}{$(\alpha,\beta)=(0.5,0)$}  & 0.73 & 0.25 & 1.98& $3.33\times10^{-16}$  \\ \hline
\end{tabular}}
\end{table}
In Table~\ref{tab:external} we compare the external organization of METABRIC using our weighted metric to the original EMD-like metric for different values of $\beta$ and $\alpha$.
Our results show that the data-driven weights achieve the best results, reinforcing that learning the weights in a data-adaptive way is more beneficial than setting the weights based on the size of the folders or the level of the tree.
Applying external organization enables us to assess which bi-organization of the external dataset and corresponding learned metric were the most meaningful.
Note that for some of the parameter choices ($\alpha=0$, $\beta=1$ or $\beta=-1$), the external organization of $Z_{\textrm{M}}$ using a gene tree learned from the dataset $Z_{\textrm{T}}$ was better than the internal organization. 
Thus, via the organization of the dataset $Z_{\textrm{M}}$, we validate that the hierarchical organization of the genes in $Z_{\textrm{T}}$, and therefore, the corresponding metric, are effective in clustering samples into cancer subtypes.
This also demonstrates that the hierarchical gene organization learned from one dataset can be successfully applied to another dataset to learn a meaningful sample organization, even though the two were profiled using different technologies.
This provides motivation to integrate information from datasets together.

In our final evaluation, we divide the METABRIC dataset into its two original subsets: the discovery set comprising 997 tumors and the validation set comprising 995 tumors.
Note that the two sets have different sample distributions of cancer subtypes.
We compare three approaches for organizing the data.
We begin with the self-organization as in Sec.~\ref{sec:res_self}.
We organize each of the two datasets separately and report their clustering measures in the first row in Table~\ref{tab:disc} for the discovery cohort and in Table~\ref{tab:val} for the validation cohort.
Note that the organization achieved using half the data is less meaningful in terms of the survival rates compared to using all of the data.
This is due to the different distribution of subtypes and survival times between the discovery and validation cohorts, and in addition, the p-value calculation itself is dependent on the sample size used.

One of the important aspects in a practical application is the ability to process new samples. 
Our approach naturally allows for such a capability.
Assume we have already performed bi-organization on an existing dataset and we acquire a few new test samples.
Instead of having to reapply the bi-organization procedure to all of the data, we can instead insert the new samples into the existing organization.
We demonstrate this by using each subset of the METABRIC dataset to organize the other.
In contrast to the external organization example, here we have two datasets profiled with the same technology.
We can treat this as a training and test set scenario: construct a sample tree on the training set $\Y_{\textrm{train}}$ and use the learned metric on the genes $d_{\T_\X}$ to insert samples from the test set $\Y_{\textrm{test}}$ into the training sample tree $\T_{\Y_{\textrm{train}}}$.
First, we calculate the centroids of the folders $\mathcal{J}_j$ of level $l=1$ (the level above the leaves) in the samples tree $\T_{\Y_{\textrm{train}}}$:
\begin{equation}
 C_{j}(x) = \sum_y \mybf{M}_\Y[j,y] \mybf{Z}[x,y], \;\; x\in\{1,..., n_\X\}, \;\;  l(\mathcal{J}_j)=1
\end{equation}
These can be considered the representative sample of each folder.
We then assign each new sample $y \in \Y_{\textrm{test}}$ to its nearest centroid using the metric $d_{\T_\X}(y,C_{j})$ derived from the gene tree $\T_\X$.
Thus, we reconstruct the sample hierarchy on the test dataset $\Y_{\textrm{test}}$ by assigning each test sample to the hierarchical clustering of the low-level centroids from the training sample tree. 
This approach, therefore, validates the sample organization as well as the gene organization, whereas the external organization only enables to validate the gene organization. 

We perform this once treating the validation set as the training set and the discovery set as the test set, and then vice-versa.
We report the clustering measures in the second row of Table~\ref{tab:disc} and Table~\ref{tab:val}.
Note that the measures are reported only for the samples belonging to the given set in the table.
Inserting samples from one dataset into the sample tree of another demonstrates an improved organization in some measures compared to performing self-organization.
For example, the organization of the discovery set via the validation tree results in a clustering with improved ARI and VI measures.
This serves as additional evidence for the importance of integrating information from several datasets together.

Thus far in our experiments, we have gathered substantial evidence for the importance of information stemming from multiple data sets. 
Here, we harness the multiple tree metric~(\ref{eq:multimetric}) to perform integration of datasets in a more systematic manner.
We generalize the external organization method to several datasets, where we integrate all the learned trees on the genes $\{\T_\X\}$ into a single metric via the multi-tree metric.

In addition to the gene tree from both METABRIC datasets, we also obtain the gene trees from the TCGA and the BRCA-547 datasets, $Z_{\textrm{T}}$ and $Z_{\textrm{B}}$. 
We then calculate a multi-tree metric~(\ref{eq:multimetric}) to construct the sample tree on either the discovery or validation sets.
We report the evaluation measures in the third row of Table~\ref{tab:disc} and Table~\ref{tab:val}.
Taking into account all measures, the multi-tree metric incorporating four different datasets best organizes both the discovery and validation datasets. 
Integrating information from multiple sources improves the accuracy of the organization, as averaging the metrics emphasizes genes that are \emph{consistently} grouped together, representing the intrinsic structure of the data.
In addition, since the metric integrates the organizations from several datasets, it is more accurate than the internal organization of a dataset with few samples or a non-uniform distribution of subtypes.

Our results show that external organization, via either both single or multi-tree metric, enables us to learn a meaningful multi-scale hierarchy on the genes and apply it as a metric to organize the samples of a given dataset.
Thus, we can apply information from one dataset to another to recover a multi-scale organization of the samples, even if they were profiled in a different technique.
In addition, we obtain a validation of the gene organization of one dataset via another.
This cannot be accomplished with traditional hierarchical clustering in a clustered dendrogram as the clustering of the samples does not depend on the hierarchical structure of the genes dendrogram.
However, we can obtain an iterative hierarchical clustering algorithm for biclustering using our approach.
As our bi-organization depends on a partition tree method, we can use hierarchical clustering instead of flexible trees in the iterative bi-organization algorithm. 
Alternatively, as hierarchical clustering depends on a metric, this can also be formulated as deriving a transform from the dendrogram on the genes and using its corresponding tree-metric instead of correlation as the input metric to the hierarchical clustering algorithm on the samples, and vice-versa.

In related work, Cheng, Yang and Anastassiou~\cite{Cheng2013} analyzed multiple datasets and identified consistent groups of genes across datasets. 
Zhou et al.~\cite{Zhou2005} integrate datasets in a platform independent manner to identify groups of genes with the same function across multiple datasets.
The multi-tree transform can also be used to identify such genes, however this is beyond the scope of this paper and will be addressed in future work.

\begin{table}[t]
\centering
\caption{METABRIC discovery organization}
\label{tab:disc}
\scalebox{0.9}{
\begin{tabular}{l|l|l|l|l|}
\cline{2-5}
discovery                                                                                      & RI   & ARI  & VI   & p-value              \\ \hline
\multicolumn{1}{|l||}{Self-organization}                                                        &\bf{0.75} & 0.33 & 1.81 & $1.82\times10^{-11}$ \\ \hline
\multicolumn{1}{|l||}{\begin{tabular}[c]{@{}l@{}}Inserted into \\ validation tree\end{tabular}} & 0.74 & 0.34 & 1.66 & $2.93\times10^{-9}$  \\ \hline
\multicolumn{1}{|l||}{Multi-tree}                                                                & \bf{0.75} & \bf{0.35} & \bf{1.63} & $\bf{3.18\times10^{-13}}$ \\ \hline
\end{tabular}}
\end{table}

\begin{table}[t]
\centering
\caption{METABRIC validation organization}
\label{tab:val}
\scalebox{0.9}{
\begin{tabular}{l|l|l|l|l|}
\cline{2-5}
validation                                                                                    & RI   & ARI  & VI   & p-value             \\ \hline
\multicolumn{1}{|l||}{Self-organization}                                                       & \bf{0.77} & 0.33 & 1.82 & $3.07\times10^{-4}$ \\ \hline
\multicolumn{1}{|l||}{\begin{tabular}[c]{@{}l@{}}Inserted into \\ discovery tree\end{tabular}} & 0.76 & 0.30 & 1.98 & $9.08\times10^{-8}$ \\ \hline
\multicolumn{1}{|l||}{Multi-tree}                                                               & 0.76 & \bf{0.34} & \bf{1.73} & $\bf{4.24\times10^{-9}}$ \\ \hline
\end{tabular}}
\end{table}

\subsection{Sub-type labels}
In breast cancer, PAM50~\cite{Parker2009} is typically used to assign intrinsic subtypes to the tumors. 
However, Milioli et al.~\cite{Milioli2016} recently proposed a refined set of subtypes labels for the METABRIC dataset, based on a supervised iterative approach to ensure consistency of the labels using several classifiers. 
Their labels are shown to have a better agreement with the clinical markers and patients' overall survival than those provided by the PAM50 method.
Therefore, the clustering measures we reported on the METABRIC dataset were with respect to the refined labels.

Our unsupervised analysis demonstrated higher consistency with the refined labels than with PAM50.
Thus, our unsupervised approach provides an additional validation to the labeling achieved in a supervised manner.
We divided the data into training and test sets and classified the test set using k-NN nearest neighbors with majority voting using the tree-based metric.
For different parameters and increasing numbers of genes ($n_\X=500,1000,2000$), we had higher agreement with the refined labels than with PAM50, achieving a classification accuracy of 82\% on average. 
Classifying with the PAM50 labels had classification accuracy lower by an average of $10\%\pm2\%$.
This is also evident when examining the labels in Fig.~\ref{fig:biorg_init}.
Note that whereas PAM50 assigns a label based on 50 genes and the refined labels were learned using a subset of genes found in a supervised manner, our approach is unsupervised using the $n_\X$ genes with the highest variance.

\section{Conclusions}
In this paper we proposed new data-driven tree-based transforms and metrics in a matrix organization setting.
We presented partition trees as inducing a new multiscale transform space that conveys the smooth organization of the data, and derived a metric in the transform space. 
The trees and corresponding metrics are updated in an iterative bi-organization approach, organizing the observations based on the multiscale decomposition of the features, and organizing the features based on the multiscale decomposition of the observations.
In addition, we generalized the transform and the metric to incorporate multiple partition trees on the data, allowing for the integration of several datasets.
We applied our data-driven approach to the organization of breast cancer gene expression data, learning metrics on the genes to organize the tumor samples in meaningful clusters of cancer sub-types.
We demonstrated how our approach can be used to validate the hierarchical organization of both the genes and the samples by taking into account several datasets of samples, even when these datasets were profiled using different technologies.
Finally, we employed our multi-tree metric to integrate information from the organization of these multiple datasets and achieved an improved organization of tumor samples.

In future work, we will explore several aspects of the multiple tree setting.
First, the multi-tree transform and metric can be incorporated in the iterative framework for further refinement. 
Second, we will generalize the coherency measure to incorporate multiple trees.  
Third, we will apply the multi-tree framework to a multi-modal setting, where observations are shared across datasets, as for example, in the joint samples shared by the BRCA-547 and TCGA datasets.
Finally, we will reformulate the iterative procedure as an optimization problem, enabling to explicitly introduce cost functions.
In particular, cost functions imposing the common structure of the multiple trees across datasets will be considered.

\section*{Acknowledgments}
The authors thank the anonymous reviewers for their constructive comments and useful suggestions.

\bibliographystyle{IEEEtran}
\bibliography{mybib}

\appendices
\section{Flexible trees} 
\label{app:partition}
We briefly describe the flexible trees algorithm, given the feature set $\X$ and an affinity matrix on the features denoted $\mathbf{K}_\X$.
For a detailed description see~\cite{Ankenman2014}. 
\begin{enumerate}
\item Input: The set of features $\X$, an affinity matrix $\mathbf{K}_\X\in \mathbb{R}^{n_\X \times n_\X}$, and a constant $\epsilon$.
\item Init: Set partition $\mathcal{I}_{0,i} = \{i\} \; \forall \; 1 \leq i \leq n_\X$, set $l=1$.
\item Given an affinity on the data, we construct a low-dimensional embedding on the data~\cite{Coifman2006}.
\item \label{item:dist} Calculate the level-dependent pairwise distances $d^{(l)}(i,j) \; \forall \; 1 \leq i,j \leq n_\X$ in the embedding space.
\item Set a threshold $\frac{p}{\epsilon}$, where $p=\textrm{median}\left(d^{(l)}(i,j)\right)$.
\item For each index $i$ which has not yet been added to a folder, find its minimal distance $d^{\min}(i)=\min_j\{d^{(l)}(i,j)\}$.
\begin{itemize}
\item If $d^{\min}(i)<\frac{p}{\epsilon}$, $i$ and $j$ form a new folder if $j$ does not belong to a folder. 
If $j$ is already part of a folder $\mathcal{I}$, then $i$ is added to that folder if $d^{\min}(i)<\frac{p}{\epsilon} 2^{-\vert I \vert + 1}$. 
\item If $d^{\min}(i) > \frac{p}{\epsilon}$, $i$ remains as a singleton folder.
\end{itemize}
\item \label{item:partition} The partition $\mathcal{P}_l$ is set to be all the formed folders.
\item For $l>1$ and while not all samples have been merged together in a single folder, steps \ref{item:dist}-\ref{item:partition} are repeated for the folders $\mathcal{I}_{l-1,i} \in \mathcal{P}_{l-1}$. 
The distances between folders depend on the level $l$, and on the samples in each of the folders.
\end{enumerate}

\section{Comparing Survival Curves}
\label{app:KM}
The survival function $S(t)$ is defined as the probability that a subject will survive past time t.
Let $T$ be a failure time with probability density function $f$. 
The survival function is
$S(t) = P(T>t),$
where the Kaplan-Meier method~\cite{Kaplan1958} is a non-parametric estimate given by
\begin{multline}
\hat{S}(t_j)= \prod_{i=1}^j Pr(T>t_i \vert T \geq t_i) = \\
\hat{S}(t_{j-1}) Pr(T>t_j \vert T \geq t_j).
\end{multline}
Defining $n_{i}$ as the number at risk just prior to time $t_{i}$ and $d_i$ as the number of failures at $t_i$, then $P(T>t_i) = \frac{n_i-d_i}{n_i}$. 
For more information on estimating survival curves and taking into account censored data see~\cite{Klein2005}

Comparison of two survival curves can be done using a statistical hypothesis test called the log-rank test~\cite{Peto1972}. 
It is used to test the null hypothesis that there is no difference between the population survival curves (i.e. the probability of an event occurring at any time point is the same for each population). 
Define $n_{k,i}$ as the number at risk in group $k$ just prior to time $t_{i}$, such that $n_i = \sum_k n_{k,i}$ and $d_{k,i}$ as the number of failures in group $k$ at time $t_{i}$ such that $d_i = \sum_k d_{k,i}$.
Then, the expected number of failures in group $k={1,2}$ is given by
 \begin{equation}
E_k =  \sum_i d_{i} \frac{n_{k,i}}{n_{i}}
\end{equation}
and the observed number of failures in group $k={1,2}$ is 
\begin{equation}
O_k = \sum_i d_{k,i}.
\end{equation}

Under the null hypothesis of no difference between the two groups, the log-rank test statistic is
  \begin{equation}
 \frac{(O_2 - E_2)^2}{\textrm{Var}(O_2 - E_2 )} \sim \chi_1^2.
 \end{equation}
The log-rank test can be extended to more than two groups~\cite{Klein2005}.

\vfill\pagebreak

\end{document}